\documentclass{ecai}
\usepackage{graphicx}
\usepackage{latexsym}

\newtheorem{Theorem}{Theorem}

\newtheorem{Lemma}{Lemma}

\usepackage{amsmath}

\begin{document}

\begin{frontmatter}

\title{Double logistic regression approach to biased positive-unlabeled data}

\author[A]{\fnms{Konrad}~\snm{Furmańczyk}}\orcid{0000-0002-7683-4787}
\author[B,D]{\fnms{Jan}~\snm{Mielniczuk}}
\author[C]{\fnms{Wojciech}~\snm{Rejchel}}
\author[B,D]{\fnms{Paweł}~\snm{Teisseyre}\thanks{Corresponding Author. Email: teisseyrep@ipipan.waw.pl}} 

\address[A]{University of Life Sciences,
Warsaw, Poland}
\address[B]{Polish Academy of Sciences, Warsaw, Poland}
\address[C]{Nicolaus Copernicus University, Toruń, Poland.}
\address[D]{Warsaw University of Technology, Warsaw, Poland}

\begin{abstract}
Positive and unlabelled learning is an important non-standard inference problem which arises naturally in many applications. The significant limitation of almost all existing methods  addressing it lies in assuming that the propensity score function is constant and does not  depend on features (Selected Completely at Random assumption), which is unrealistic in many practical situations.  
  Avoiding this assumption, we consider parametric approach to the problem of joint estimation of posterior probability and propensity score functions.
 We show that if  both these  functions are logistic with different parameters (double logistic model) then the corresponding parameters are  identifiable.
 Motivated by  this, we propose  two approaches to their estimation:  a joint maximum likelihood  method and the second approach based on an alternating maximization of two Fisher consistent approximations. 
 Our experimental results show that the proposed methods  perform on par or better than the existing methods based on Expectation-Maximisation scheme.
\end{abstract}

\end{frontmatter}

\section{Introduction}

Positive unlabelled (PU) inference is based on data sets containing  labelled observations $(S=1)$ which are all positive ($Y=1$),  and unlabelled ones ($S=0$) which may  either  belong to a positive or a negative  class ($Y$ is either 1 or 0). Examples of such experimental setup abound in medicine \cite{Walleyetal2018, LAN2016, Cerulo2010, Yangetal2014}, text and image analysis \cite{Fung2006, Liu2003, LiLiu2003,Jain2018}, ecology \cite{Ward2009, PearceBoyce2006} and survey  data \cite{SECHIDIS2017}.  For example, medical databases may contain only information about diagnosed patients who  have a certain disease $(S=1)$ whereas un-diagnosed patients $(S=0)$ may have it or not. In survey sampling,  asking a sensitive question (e.g. on   use of illicit drugs) may lead to under-reporting, as beside the positive respondents $(Y=1)$ who answer the question truthfully,  there are respondents who engage in this activity and do not  admit  it  ($Y=1,S=0)$. Their answers are   merged together with those of people who abstain from  such behaviour and  answer the question negatively $(Y=0,S=0)$ \cite{Bahorik2014}.
 PU data  occur frequently in  text classification problems. For example, when classifying web page preferences, some web pages can be bookmarked as positive ($S=1$) by the user  whereas all other pages are treated as unlabelled ($S=0$). Among unlabelled pages ($S=0$), one can find both positive and negative pages. The other important example is associated with detecting unlawful content in social networks. In this case, certain content has been marked as unlawful (e.g. some images or posts), however unlawful content may still exist among the unmarked profiles.

In the seminal paper \cite{ElkanNoto2008}  an influential  approach to this problem which is  proposed based on assumption that probability of labelling of positive elements is not instance dependent,
i.e. $P(S=1\vert Y=1,x)=P(S=1\vert Y=1)$ (Selected Completely at Random, or SCAR, assumption), where $x$ is a feature vector and constant 
$c=P(S=1\vert Y=1)$ is called label frequency. 
For a review of the developments, almost exclusively based on SCAR,  we refer to \cite{BekkerDavis2018}. The SCAR assumption facilitates inference significantly, as in this case, aposteriori probability $P(Y=1 \vert x)$, which is often of the main interest, can be written as $P(Y=1 \vert x)=c^{-1}P(S=1 \vert x)$, where $P(S=1 \vert x)$ can be estimated using the observed PU data. 
In view of this, estimation of $c$ becomes a crucial problem 
\cite{ElkanNoto2008, Ramaswamy2016, Jainetal2016, Plessis2017, BekkerAAAI18, JaskieElkan2020, Lazeckaetal2021}. The common approach here is to treat first the unlabelled observations as coming from negative ($Y=0$) class and then detect among them those observations, which, due to their  covariates' closeness to  the labelled data, should be assigned to the positive class (see e.g. \cite{Sansone2018}, \cite{zhang2019}). Other important approaches are based on suitable modification of the risk function using weighting to account for unobservability of negative examples (see \cite{duPlessis2014} and \cite{Kiryo2017}).

However, the SCAR assumption   fails  in many practical situations. For example, an age is  an important factor in
screening for  many diseases (such as a prostate cancer) \cite{Leitzmann2012} which may lead to a positive diagnosis.
Moreover, the occurrence of other diseases (e.g. obesity), may  play a role in undertaking  an   in-depth scrutiny for other
potential illnessess (e.g. diabetes) \cite{Vistisen2014}.
In surveys, the criminal background of the interviewee is a strong indication  that  obtaining an untrustworthy answer is likely. In a general case, the situation is much more complicated than under SCAR, as $P(S=1\vert x)$ may be small even if $P(Y=1\vert x)$ is large. 
Importantly, ignoring the fact that the probability of labelling depends on features will lead to biased estimation of the posterior probability. An accurate estimation of this probability enables  a precise estimation of the posterior probability which, in its turn,  leads to  an accurate prediction. 
The problem draws more attention recently \cite{bekker2019ecml,Gong2021, Gerych2022}. 
Besides approaches   based on EM algorithms (see below), other methods   to tackle this problem are based  on the concept of probabilistic gap (\cite{He2018}),  assumptions that ordering of posterior and propensity score with respect to $x$ coincide (\cite{kato2019}) or on application of  deep learning techniques (\cite{Na2020}).
 In a broader context, the situation when elements of underlying sample  were chosen taking the values of their covariates into account  is frequently termed labelling (selection) bias or covariate shift, and its importance is recognised by many authors (\cite{huang2006}, \cite{kesmodel2018}).

The restrictiveness of SCAR assumption calls for  functional modelling of the probability of being labelled which corresponds to propensity score in causal inference. The  important steps in this direction has been taken recently in  \cite{bekker2019ecml} and \cite{Gong2021}, where variants of Expectation-Maximisation (EM) algorithm have been considered,  see Section \ref{related} for the detailed discussion of the methods. The present paper  also addresses this issue and contains   the following new developments: firstly, we consider parametric models for propensity score $e(x)=P(S=1\vert Y=1,x)$ and $P(Y=1\vert x)$ and show that their parameters are identifiable given values of $P(S=1\vert x)$ only. In particular, the both functions can be modeled as logistic  functions with different parameters; this setup will be called double logistic model in the following.
This naturally  leads to  an introduction of  a joint maximum likelihood (ML) estimators of these parameters and  establishing their consistency  when the model is well specified.  
We note that the  analysis of this parametric approach is hindered by the fact that even in the SCAR case log-likelihood {\it is not necessarily  a concave function} of the underlying parameters.
Secondly, we introduce a method (called 'Two MODELS' method, 'TM' in brief) consisting in alternate maximising  {\it concave} empirical surrogates for  expected likelihoods of posterior probability of $Y=1$ and the propensity score. We  prove  that the method is consistent under some simplifying assumptions. Moreover, we  numerically investigate the behaviour of our  proposals and show that the TM method consistently exhibits superior   or comparable behaviour to the best of existing methods \cite{bekker2019ecml, Gong2021}.

The rest of the paper is structured as follows. In Section \ref{prelims} we formally describe the PU learning problem and define basic quantities. In Section \ref{Sec:Joint estimation of posterior probability and propensity score function} we discuss the problem of joint estimation of posterior probability and propensity score functions and state main theoretical results. The algorithms (including the proposed ones) are described in Section  \ref{Sec:Algorithms}, whereas the two existing most related methods (EM and LBE) are discussed in Section \ref{related}. In Section \ref{Sec:Numerical experiments} we describe the results of experiments and in Section \ref{Sec:Conclusions} we conclude our work. The Appendix  
available at \url{https://github.com/teisseyrep/putm}
contains the proofs and some additional numerical results. 

\section{Background}
\label{prelims}
We first introduce basic notations. Let $X$ be a random variable corresponding to feature vector, $Y\in\{0,1\}$ be a true class label and $S\in\{0,1\}$ an indicator of  an example being labelled ($S=1$) or not ($S=0$). We assume that there is some unknown distribution $P_{Y,X,S}$ such that  
$(Y_i,X_i,S_i), i=1,\ldots,n$ is iid sample drawn from it. Observed data  consists of $(X_i,S_i),i=1,\ldots,n$ (so called, the single sample scenario).  
Only positive examples ($Y=1$)  can be labelled, i.e. $P(S=1\vert X,Y=0)=0$.  Thus we know that $Y=1$ when $S=1$ but when $S=0$, label  $Y$ can be either 1 or 0.
Our aim    is to learn binary posterior distribution of $Y$ given $X=x$ i.e. $y(x)=P(Y=1\vert X=x)$  and we only   observe samples from distribution of $(X,S)$, where $S=Y$ with a certain probability.
To this end we define a  binary {\it posterior distribution  function of  $S$  given $x$} as $s(x)=P(S=1\vert x)$ and {\it a propensity score function} $e(x)=P(S=1\vert Y=1,x)$. We note that
\begin{equation}
 \label{posteriorS}  
 s(x)=e(x)y(x)
 \end{equation}
 as $P(S=1\vert Y=0,x)=0$. 
 In the following we assume that $e(x)$ may depend on $x$ that is we do not impose restrictive and hard to verify Selected Completely at Random (SCAR) assumption. 
We note that SCAR assumption implies that the distribution of $X$ for labelled data coincides with its distribution in the positive class, but this is not true in general. This makes inference much harder task in the general setting as the distribution of the labelled data is biased.
We note that $e(x)$ plays a role of a nuisance functional parameter and our primary objective is to estimate $y(x)$.

We stress that in parallel to the single sample scenario, the case-control  (c-c) scenario is frequently considered for PU data. In this scenario in addition to the labelled data from the positive class we have at our disposal unlabelled sample drawn from {\it the marginal distribution} of $X$. The form of the second sample makes the  inference problem   different (and in general easier) than that studied here. Moreover, solutions obtained for c-c case are not transferable to the single sample case.  For approaches developed for c-c case see e.g. \cite{Kiryo2017} and \cite{kato2019}.

\section{Joint estimation of posterior probability and propensity score function}
\label{Sec:Joint estimation of posterior probability and propensity score function}

As only $s(x)$ is observable and $e(x)$ is  an unknown function which  is not constant, identification of posterior $y(x)$ in view of  (\ref{posteriorS}) is clearly impossible in general. However, we will show that if certain parametric assumptions are imposed on $y(x)$ and $e(x)$ then both  functions are identifiable up to an interchange of $y(x)$ and $e(x)$. Namely, let $\sigma(s)=1/(1+e^{-s})$ be a logistic function and assume that both $y(x)$ and $e(x)$ are governed by the logistic model:
\begin{equation}
 \label{param}  
 y(x)=\sigma(\beta_0^* +\beta^{*T}x)\quad\quad e(x)=\sigma(\gamma_0^* +\gamma^{*T}x).
 \end{equation}
 We will call PU model for which (\ref{param}) is satisfied a double logistic model. Note that no assumptions on the distribution of the vector of features $X$ is imposed.
 As logistic model is quite versatile, it is not unrealistic to assume, that in many situations both $y(x)$ and $e(x)$ may follow it, at least approximately. We also note that Two Models method  proposed below can be  combined with other classifiers (e.g. neural networks), but this is left for a future research.\\
 For any $b_0\in R$ and $b\in R^p$ with some abuse of notation we let $\tilde b=(b_0,b^T)^T.$
We have the following result which plays an important role in proving the consistency of joint maximum likelihood estimation.

\begin{Theorem}
\label{identification}
Consider $s(x)$  defined as in (\ref{posteriorS}) and assume  that  $y(x)$ and $e(x)$ satisfy~(\ref{param}). Then parameters $\tilde\beta^*$ and $\tilde\gamma^*$ are uniquely defined up to an interchange  of $y(x)$ and $e(x)$ i.e. if for some $\tilde \beta$ and $ \tilde \gamma$ we have  $s(x)=\sigma(\beta_0 +\beta^Tx)\sigma(\gamma_0 +\gamma^Tx)$ for all $x\in R^p$, 
then   $(\tilde \beta, \tilde \gamma) = (\tilde\beta^*,  \tilde\gamma^*)$ or $(\tilde \beta, \tilde \gamma) =(\tilde\gamma^*,\tilde\beta^*). $
\end{Theorem}

The  proof of Theorem 1 is  contained  in the Appendix 1. Moreover, we show in  Theorem 5 in the  Appendix 1 that this result actually holds  for a general function  response  $p$ replacing $\sigma$ in definitions of $y(x)$ and $e(x)$ under certain mild assumptions imposed on the logarithmic derivative $p'(s)/p(s)$.

Assume now that the logistic model is fitted both to $y(x)$ and $e(x)$ and consider the risk function corresponding to logistic loss for $(X,S)$
\begin{equation}
\label{risk}
  Q(\tilde\beta,\tilde\gamma)=
  E_{X,S}[S\log s_{\tilde\beta,\tilde\gamma}(X) + (1-S)\log (1-s_{\tilde\beta,\tilde\gamma}(X)) ],  
\end{equation}
where  $s_{\tilde\beta,\tilde\gamma}(x)=\sigma(\beta_0 +\beta^Tx)\sigma(\gamma_0 +\gamma^Tx)$. We have the following result in which $\vert\tilde b\vert_1=\sum_{i=0}^ p \vert b_i\vert$ for $\tilde b=(b_0,b_1,\ldots,b_p)^T$ denotes the $l_1$ norm.
\begin{Lemma}
\label{KL}
Let assumptions of Theorem \ref{identification} hold and $\vert\tilde\beta^*\vert_1 >\vert\tilde\gamma^*\vert_1 $. Then 
\[
(\tilde\beta^{*T},\tilde\gamma^{*T})^T=  \arg \max_{(\tilde\beta,\tilde \gamma): \vert\tilde\beta\vert_1 >\vert\tilde\gamma\vert_1}  Q(\tilde\beta,\tilde\gamma) 
\]
and $(\tilde\beta^{*T},\tilde\gamma^{*T})^T$ is the unique maximiser of $Q(\tilde\beta,\tilde\gamma). $ 
\end{Lemma}
The assumption $\vert\tilde\beta^*\vert_1 >\vert\tilde\gamma^*\vert_1 $ is imposed due to the possibility of the fact that $Q$ is a symmetric function: $Q(\tilde\beta,\tilde\gamma)= Q(\tilde\gamma,\tilde\beta)$.   We note that the  $l_1$ norm in this  condition   is not essential and may be replaced by  any norm.
The proof of the  Lemma \ref{KL} is relegated to the Appendix 1.


Define an empirical counterpart of $ Q(\tilde\beta,\tilde\gamma)$ given in (\ref{risk}) as
\begin{equation*}
    Q_n(\tilde\beta,\tilde\gamma)=
   \frac{1}{n}\sum_{i=1}^n [S_i\log s_{\tilde\beta,\tilde\gamma}(X_i) + (1-S_i)\log(1- s_{\tilde\beta,\tilde\gamma}(X_i))].
\end{equation*}
In the view of  Lemma \ref{KL} it is intuitive to expect that maximisers of $ Q_n(\tilde\beta,\tilde\gamma)$ will approximate true parameters $\tilde\beta^*$ and $\tilde\gamma^*$ of the generating mechanism. Indeed, we have the following result.
\begin{Theorem}
\label{SC}
(Strong consistency of joint ML estimation) Let assumptions of Lemma~\ref{KL} hold and $K((\tilde\beta^*,\tilde\gamma^*),r)$ be a closed ball with the centre  $(\tilde\beta^*,\tilde\gamma^*)$ and a radius $r>0.$ Suppose that for each $x$ and 
$(\tilde\beta,\tilde\gamma) \in K((\tilde\beta^*,\tilde\gamma^*),r)$ functions $\log s_{\tilde\beta,\tilde\gamma} (x) $ and $\log (1-s_{\tilde\beta,\tilde\gamma} (x))$ are bounded from below by a function $\eta (x)$ such that $E \vert\eta(X)\vert <\infty.$ 
Then with probability one,  for sufficiently large $n$ there exists a sequence $(\hat \beta_n, \hat \gamma_n)$ of local  maximisers of $Q_n(\tilde \beta, \tilde \gamma)$ such that $(\hat \beta_n, \hat \gamma_n) \rightarrow (\tilde\beta^*,\tilde\gamma^*)$.
\end{Theorem}

Note that  the assumption $E\vert\eta(X)\vert<\infty$ imposed in Theorem \ref{SC} does not force $s(x)$ to  be bounded way from 0 and 1 which is frequently assumed while dealing with consistency issues of estimates in the logistic model.\\
Finding the global maximiser of $Q_n(\tilde\beta,\tilde\gamma)$ defined above is a complicated  task as the optimised function is not  concave  in either $\tilde\beta$ or $\tilde\gamma$; see e.g. \cite{Lazeckaetal2021}, where it is shown that $Q_n$ is not concave even under SCAR  when $e(x)$ is assumed constant. Thus we also introduce here a second approach which consists in iterative alternate solving for maxima  of {\it concave} empirical likelihoods of $y(x)$ and $e(x).$

Let $y(x,\tilde\beta)=\sigma(\beta_0 +\beta^Tx)$ and $e(x,\tilde\gamma)= \sigma(\gamma_0 +\gamma^Tx)$.
 We  will  thus look for solutions of empirical counterparts of two optimisation problems.
Optimisation problem for $y(x,\tilde\beta)$ is to maximise wrt  $\tilde \beta$
\begin{eqnarray}
\label{loglik}
&&
E_X W(X,\tilde \beta)=
E_X[
y(X,\tilde\beta^*)
\log y(X,\tilde\beta) + 
\cr
&&
(1-y(X,\tilde\beta^*))
\log(1-y(X,\tilde\beta))], 
\end{eqnarray}
where $W(X,\tilde \beta)$ is the bracketed expression above. Note that (\ref{loglik}) is the expected value of the loglikelihood of $(Y,X)$ in the double logistic model.
Let $K(s,x,\tilde\gamma)=s \log e(x,\tilde\gamma) +  (1-s)\log(1-e(x,\tilde\gamma))$.
and  note $E_{S\vert Y=1,X} S=e(X,\tilde \gamma ^*)$.
Using this equality, we note that optimisation problem for $e(x,\tilde\gamma)$ can be approached via  maximising wrt to~$\tilde \gamma$
\begin{eqnarray}
\label{ps}    
&&
 E_{X\vert Y=1}[
e(X,\tilde\gamma^*)
\log e(X,\tilde\gamma) + 
\cr
&&
(1-e(X,\tilde\gamma^*))
\log(1-e(X,\tilde\gamma))]= 
\cr
&&
E_{X\vert Y=1} E_{S\vert Y=1,X} K(S,X,\tilde\gamma)=
\cr
&& 
E_{S,X\vert Y=1} K(S,X, \tilde \gamma),
\end{eqnarray}

Notice that  in the case of (\ref{loglik}) for any $X=x$ maximiser of $W(x,\tilde \beta)$ is $\tilde\beta^*$ (this can be seen reasoning analogously as in the case of
Lemma~\ref{KL}), whereas in the case of  (\ref{ps}) maximiser of $E_{S,X\vert Y=1} K(S,X, \tilde \gamma)$ is  $\tilde\gamma^*$. Thus  $E_X W(X,\tilde \beta)$ and  $E_{S,X \vert Y=1} K(S,X, \tilde \gamma)$ are Fisher consistent  in the sense that maximisation over $\tilde\beta$ and $\tilde\gamma$ yields true parameters $\tilde\beta^*$ and $\tilde\gamma^*$, see \cite{LiDuan1989} for discussion of Fisher consistency. This is an important property as Fisher consistency  implies strong consistency of empirical maximisers under mild assumptions.
The obvious problem is that  neither (\ref{loglik}) nor (\ref{ps}) have direct empirical counterparts due to  dependence on 
 $y(x,\tilde \beta^*)$ 
in the first case and in the second case due to averaging over the unknown conditional distribution of $(S,X)$ given $Y=1$. In case of (\ref{loglik}) we will solve this problem  by introducing weights depending on 
propensity score such that the weighted risk based on these weights will equal $E_X W(X,\tilde \beta)$. Then using the current estimator of 
$e(x, \tilde \gamma ^*)$ 
we will define an approximation to its empirical counterpart which will be maximised. The obtained estimator of posterior probability will be used to approximate the stratum $\{Y=1\}$ and the expected value with respect to $S,X\vert Y=1$ and  thus making evaluation of empirical counterpart of (\ref{ps}) feasible.
Namely, for  the first problem we want to find weights $w_1(s,x)$ and $w_0(s,x)$ such that 
\begin{eqnarray}
\label{weights}
&&
W(x, \tilde \beta)= E_{S\vert X=x}[w_1(S,x)\log y(x,\tilde\beta)+ 
\cr
&&
  w_0(S,x)\log(1-y(x,\tilde\beta))].
\end{eqnarray}
 Then maximising  an empirical counterpart of (\ref{weights}) yields consistent estimator of $\tilde\beta^*$.
In the case of the second optimisation we have to approximate expectation $E_{S,X\vert Y=1}$.  For the first problem we have
\begin{Lemma}
\label{L2}
Let $w_1(S,x)= I\{S=1\} + I\{S=0\}P(Y=1\vert S=0,x)$ and $w_0(S,x)=I\{S=0\}P(Y=0\vert S=0,x)$. Then (\ref{weights}) holds.
\end{Lemma}
{\textit Proof.} 
Observe that   $W(x, \tilde \beta)$ equals
\begin{align*}
[ P(S=1\vert x) + P(S=0\vert x)P(Y=1\vert S=0,x)]\log y(x,\tilde\beta)] \\
+[P(S=0\vert x)P(Y=0\vert S=0,x)]\log(1-y(x,\tilde\beta))] \nonumber
\end{align*}
and the bracketed terms are equal $P(Y=1\vert x)$ and $P(Y=0\vert x)$, respectively. 

The weights $w_i(s,x)$ were introduced in \cite{ElkanNoto2008}. The lemma  above states  that they yield unbiased estimator of $W(x, \tilde \beta)$.
Note that the factor $P(Y=1\vert S=0,x)$  appearing in $w_1(S,x)$ equals 
\[
OR(x)=\frac{1-e(x)}{e(x)}/\frac{1-s(x)}{s(x)}
\]
and thus is  the odds ratio equal to the ratio of the odds of being unlabelled  among positive observations and the odds of being unlabelled in the general population.
In the view of  this and (\ref{weights}) the empirical counterpart of $E_XW(X, \tilde \beta)$ is defined as
\begin{eqnarray}
 \label{empW}
&&
 W_n(\tilde\beta)=
 \frac{1}{n}\sum_{i=1}^n \hat w_1(S_i,X_i)\log 
 y(X_i, \tilde \beta) + 
\cr
&& 
  \hat w_0(S_i,X_i)\log (1- y(X_i, \tilde \beta) ),
\end{eqnarray}
where $\hat w_1(S_i,X_i)= I\{S_i=1\} + I\{S_i=0\}\widehat{OR}(X_i) $, $\hat w_0(S_i,X_i)= I\{S_i=0\}(1-\widehat{OR}(X_i)) $  and $\widehat{OR}(x)=\frac{1-\hat e(x)}{\hat e(x)}/\frac{1-\hat s(x)}{\hat s(x)}$  
 is an estimator of $OR(x),$ which  is 
discussed in Subsection \ref{Sec:TM}. 
We will prove below that if $OR(x)$ is consistently estimated, then any maximiser of $W_n(\cdot)$ is consistent estimator of
 $\tilde \beta^*.$ 
 Indeed, notice that the key assumption \eqref{keyf} below is satisfied if $\sup_x \vert \widehat{OR}(x) - OR(x)\vert  \rightarrow_P 0.$
 \begin{Theorem}
 \label{th_beta}
Let $\tilde \beta ^*$ be the unique maximiser of $E_X W(X,\tilde \beta)$ and for each $\tilde \beta$
\begin{eqnarray}
\label{keyf}
&&
\frac{1}{n} \sum_{i=1}^n \widehat{OR}(X_i) I(S_i=0)( \beta^T X_i +\beta_0) \rightarrow_P 
\cr
&&
E \left[OR(X)I(S=0) (\beta^TX +\beta_0)\right].
\end{eqnarray}
Then every $\hat \beta_n=\arg \max_{\tilde \beta} W_n(\tilde \beta)$ tends to $\beta^*$ in probability. 
 \end{Theorem}

The proof of Theorem \ref{th_beta} can be found in the Appendix 1.

We consider now the second problem that is consistent estimation of $\tilde \gamma ^*.$
Let $n_1 = \# \{1\leq i \leq n: Y_i=1\}$ be a number of positive observations in a data set. Let $\hat Y_i$'s be some predictors of unknown $Y_i$'s and 
$\hat n_1 = \# \{1\leq i \leq n: \hat Y_i=1\}.$
 Besides, we  consider a function 
 $\hat R_n(\tilde \gamma) = \frac{1}{\hat n_1} \sum\limits_{1 \leq i \leq n: \hat Y_i=1}  K(S_i,X_i, \tilde \gamma),$
which we use to approximate
\[R(\tilde \gamma):=
E_{S,X\vert Y=1} K(S,X,\tilde \gamma)\] 
in \eqref{ps}. Finally,  changing $\hat Y_i$ to $Y_i$ in the definition of $\hat R_n(\tilde \gamma)$ we define 
$ R_n(\tilde \gamma) = \frac{1}{ n_1} \sum\limits_{1 \leq i \leq n:  Y_i=1}  K(S_i,X_i, \tilde \gamma).$
 Next, consistent estimation of $\tilde \gamma ^*$ is considered.
\begin{Theorem}
\label{th_gamma}
Let $\tilde \gamma ^*$ be the unique maximiser of $R(\tilde \gamma)$ and for each $\tilde \gamma$ we have $\hat R_n (\tilde \gamma) - R_n (\tilde \gamma) \rightarrow_P 0.$
Then every maximiser of $\hat R_n(\tilde \gamma)$ tends to $ \tilde \gamma^*$ in probability. 
 \end{Theorem}
The proof of Theorem \ref{th_gamma} can be found in the Appendix 1.

\section{Algorithms}  
\label{Sec:Algorithms}
\subsection{ NAIVE method }
\label{Sec:naive}
We first describe the NAIVE method which is the simplest approach in PU learning. In this method,  estimator of  $s(x)=P(S=1\vert X=x)$ is substituted for estimator of $y(x)=P(Y=1\vert X=x)$.  To this end  misspecified empirical loglikelihood
\begin{eqnarray*}
\sum_{i=1}^{n}S_i\log ( y(X_i,\tilde \alpha))
+(1-S_i)\log(1-
y(X_i,\tilde \alpha))
\end{eqnarray*}
is optimised with respect to $\tilde \alpha$.
The estimator of $s(x)$ is defined as 
$\hat{s}_{\textrm{naive}}(x)=y(x,\hat{\alpha})$, where $\hat{\alpha}$
is the maximizer of the above function.
Obviously, this method underestimates $y(x)$ as the positive unlabelled observations are treated as  the negative ones  and the bias increases with decreasing label frequency  $c=P(S=1\vert Y=1)$. The NAIVE method serves as a baseline in our experiments. Moreover, using the naive method, one can consider a very simple estimator of the propensity score function $e(x)$ which will serve as initial estimator in the methods described in next subsections. It is based on inequality $s(x)\leq e(x)\leq 1$ and is defined as an average of two endpoints of the interval $[\hat{s}_{\textrm{naive}}(x),1]$, i.e. $\hat{e}_{\textrm{naive}}(x)=0.5(\hat{s}_{\textrm{naive}}(x)+1)$.
\subsection{JOINT method}
We optimize function $Q_{n}(\tilde{\beta},\tilde{\gamma})$ defined in Section 3 with respect to $\tilde{\beta}$ and $\tilde{\gamma}$, alternately. We repeat the following two steps for $k=1,2,\ldots,$ until convergence:
\begin{enumerate}
\item Solve $\hat{\beta}_{n}^{(k)} = \arg\max_{\tilde{\beta}}Q_{n}(\tilde{\beta},\hat{\gamma}_{n}^{(k-1)})$. 
\item Solve $\hat{\gamma}_{n}^{(k)} = \arg\max_{\tilde{\gamma}}Q_{n}(\hat{\beta}_{n}^{(k)},\tilde{\gamma})$. 
\end{enumerate}
In the first iteration we need some initial estimator $\hat{\gamma}_{n}^{(0)}$ or equivalently initial estimator of $e(X_i,\hat{\gamma}_{n}^{(0)})$, because  $Q_{n}(\tilde{\beta},\hat{\gamma}_{n}^{(0)})$ involves  an unknown  term $e(X_i,\hat{\gamma}_{n}^{(0)})$. For this we use $\hat{e}_{\textrm{naive}}(x)$ defined in Subsection \ref{Sec:naive}. As $Q_n(\tilde\beta,\tilde\gamma)$ is not concave in either $\tilde\beta$ or $\tilde\gamma$,  Minorisation-Maximisation (MM) algorithm (see e.g. \cite{Hastie2015}, Section 5.8) is used to find the maximisers in 1 and 2. The analogous idea was used in \cite{Lazeckaetal2021, Teisseyreetal2020}, who assumed SCAR and optimized jointly with respect to $\tilde{\beta}$ and label frequency $c$. 

\subsection{ TWO MODELS method (TM)}
\label{Sec:TM}
The proposed method involves fitting two models in each iteration. The first model aims to estimate $y(x)$, whereas the second model corresponds to $e(x)$. In the case of the second model, our goal is to first approximate the stratum $\mathcal{P}:=\{i:Y_i=1\}$. We define its estimator  as $\hat{\mathcal{P}}=\{i: S_i=1 \text{ or } \hat{y}(X_i)>t\}$ where $\hat{y}(X_i)$ is an estimator of $P(Y=1\vert X_i)$ obtained from the first model and $t$ is a threshold. For the threshold we use a quantile of order $\alpha$ of the set  $\{\hat{y}(X_i) \text{ for } i \text{ such that } S_i=1\}$ (a data-adaptive choice of $\alpha$ is  discussed below). 
Next, we estimate $e(X_i)$ by fitting the logistic model using  observations  $(X_i,S_i)$ for $i\in\hat{\mathcal{P}}$.
More specifically, we repeat the following steps  until convergence:
\begin{enumerate}
\item \textbf{Model 1.} Solve $\hat{\beta}_{n} = \arg\max_{\tilde{\beta}}W_{n}(\tilde{\beta})$, where $W_{n}(\tilde{\beta})$ is defined in (\ref{empW}). 
\item Calculate 
$\hat{y}(X_i)=y(X_i, \hat{\beta}_{n}).$
\item \textbf{Model 2.} Solve $\hat{\gamma}_{n}=\arg\max_{\gamma} \hat R_{n}(\tilde \gamma)$, where
\begin{eqnarray*}
\hat R_{n}(\tilde \gamma)=\sum_{i=1}^{n}I(i\in \hat{\mathcal{P}})K(S_i, X_i, \tilde \gamma),
\end{eqnarray*}
where $K$ is defined below (\ref{loglik}).
 \item Calculate 
 $\hat{e}(X_i)=e(X_i,\hat{\gamma})$.
 \item Update $\hat{s}(X_i)=\hat{e}(X_i)\hat{y}(X_i)$ and $\widehat{OR}(X_i)=\frac{1-\hat{e}(X_i)}{\hat{e}(X_i)}/\frac{(1-\hat{s}(X_i)}{\hat{s}(X_i)}$.
\end{enumerate}
Note that in step 1, function $W_{n}(\tilde{\beta})$ depends on $\widehat{OR}(X_i)$, thus in the first iteration some  initial estimators of $s(x)$ and $e(x)$ are required. Similarly to the JOINT method, we use the naive estimators $\hat{s}_{\textrm{naive}}(x)$ and $\hat{e}_{\textrm{naive}}(x)$  described in Subsection~\ref{Sec:naive}. 
The significant advantage of TM is concavity, which allows to avoid problems with local minima. It is especially important when working with 'larger' data sets, say $p\geq 50$ and $n\geq 1000$.

An important issue is the order of the quantile $\alpha$  of $\{\hat{y}(X_i) \text{ for } i \text{ such that } S_i=1\}$  used in $\hat{\mathcal{P}}$.
Small value of $\alpha$ allows to detect significant portion of the positive observations among unlabelled ones. On the other hand,  a larger value of $\alpha$ reduces the risk of including negative examples from the set of unlabelled ones.  
The choice of the optimal value of $\alpha$ is a challenging task as it should depend on two factors:  the difficulty of the classification problem, i.e. on how much the distributions $X\vert Y=1$ and $X\vert Y=0$ overlap, as well as  on label frequency $c=P(S=1\vert Y=1)$. It follows from our experiments that when the distributions $X\vert Y=1$ and $X\vert Y=0$ are practically disjoint it is better to take a smaller $\alpha$ (for example $\alpha=0.1$), especially for small $c$.  When the distributions of $X\vert Y=1$ and $X\vert Y=0$ significantly  overlap,  large $\alpha$ is preferable (i.e. selection should become more conservative), especially when $c$ approaches $1$. In the latter case,  small $\alpha$ results in a large number of 'false positive' observations, i.e. the set $\hat{\mathcal{P}}$ contains a significant number of negative examples, which deteriorates the performance of the method. The above insights suggest that  $\alpha$ should   increase in a certain manner when label frequency increases. In our method we use $\alpha=\hat{P}(S=1)$, which is motivated by a  simple inequality $P(S=1)\leq P(S=1\vert Y=1)$. Although this choice of $\alpha$ gives very good results, we believe that the problem is worth further analysis as it is crucial for estimation of $e(x)$. 

In addition to TM method, we also consider its simplified version (called TM SIMPLE) in which we do not estimate $e(x)$ in iterative manner. Instead, we  estimate the propensity score by $\hat e_{\textrm{naive}}(\cdot)$  and then we solve $\arg\max_{\tilde{\beta}}W_{n}(\tilde{\beta})$. Comparison between TM and TM SIMPLE allows to explore the effectiveness of employing  the Model~2 in TM. In addition, TM SIMPLE is much faster than TM as it does not require running many iterations.
For example, for the largest considered dataset Adult, the average computation time is 0.5 sec for TM simple  and 11.5 sec for TM (PC Intel Core i7-10850H CPU 2.70GHz, 32.0 GB RAM).

\section{The related methods}
\label{related}
In this section we describe two existing methods, which are most related to our proposals: EM method proposed in \cite{bekker2019ecml} and LBE method proposed in \cite{Gong2021}. We keep the names of the methods according to the way they where christened by the authors, although we note (see below) that LBE is  a classical EM algorithm applied in PU setting and thus the name 'EM method' would be actually more appropriate for LBE.

LBE method relies  on parametric assumptions (\ref{param}) and is based on considering an averaged conditional likelihood for the sample $(S_i,Y_i),i=1,\ldots,n$ given $X_1,\ldots,X_n$, namely
\begin{eqnarray}
 \label{LBE}
 &&
 E_{Y_1,\ldots,Y_n\vert S_1,\ldots,S_n,X_1,\ldots,X_n}\log {\cal L}(\beta,\gamma)=
 \cr
 &&
 \sum_{i=1}^n
 E_{\tilde P(Y_i)} 
 \{\log P(Y_i\vert X_i,\beta)P(S_i\vert Y_i,X_i,\gamma)\},
 \end{eqnarray}
 where  $\tilde P(Y_i)= P(Y_i\vert X_i,S_i)$  and 
 \begin{eqnarray*}
 &&
 {\cal L}(\beta,\gamma)= 
 \cr
 &&
 P(S_1,Y_1,\ldots,S_n,Y_n\vert X_1,\ldots,X_n,\beta,\gamma) = 
\cr
&& 
 \prod_{i=1}^n P(Y_i\vert X_i,\beta)P(S_i\vert Y_i,X_i,\gamma), 
 \end{eqnarray*}
where the last equality is  due to independence of observations.
In the expectation step (E-step) binary distribution $\tilde P(Y_i)$ is estimated based on current estimates of $\beta$ and $\gamma$ using (\ref{param}), Bayes formula and a normalisation trick which is applied  to calculate estimate of $P(S_i\vert X_i)$. In the maximisation step (M-step)  current estimate of  $\tilde P(Y_i)$ is employed to calculate (\ref{LBE}) which is  then maximised using Adam (\cite{Kingma2015}) algorithm yielding the values of $\beta$ and $\gamma$ for the next E-step.

The main difference between our proposal TM and LBE method is difference in criterion functions to be optimised. In 
\cite{Gong2021} it is (\ref{LBE}),  whereas our  TM method is based on alternate maximisation of concave log-likelihoods pertaining to posterior probabilities $y(X_i)$ and propensity scores $e(X_i)$, respectively. What is more, the
theoretical analysis in  \cite{Gong2021}  does not address the identification issue studied in Theorem 1 and thus leaves the question of consistent estimation of the true vector of parameters $(\beta^{*T},\gamma^{*T})$ unanswered.

In  EM algorithm proposed in \cite{bekker2019ecml}, 
the maximization step in EM approach is similar to ours but with one crucial  difference. 
Namely, in their approach estimation of $e(x)$ is based on maximisation of estimated value of $E_{S,X,Y}YK(S,X,\gamma)= E_{S,X} P(Y=1\vert S,X)K(S,X,\gamma)$, whereas we consider maximisation of 
estimated $E_{S,X\vert Y=1} K(S,X,\gamma)$. Both expressions are Fisher consistent. Their approach leads to consideration 
of 
weights being estimates of $P(Y=1\vert S,X)$  instead of $I(i\in \hat{\cal P})$ in step~3 of TM algorithm. 
Criterion function $W_n(\tilde\beta)$  used in step 1 is the same for both methods. We show in numerical experiments below that by  using the proposed  Fisher consistent expression  for $e(x)$ and suitable approximation  of the stratum  ${\cal P}$   we obtain the method which is in most cases superior to   EM algorithm. 
 
 \section{Numerical experiments}
\label{Sec:Numerical experiments}
In the experiments we compare the following methods: NAIVE, JOINT, TM \footnote{Source code of the proposed method is available at GitHub: \url{https://github.com/teisseyrep/putm}}, TM SIMPLE (described in Section~\ref{Sec:Algorithms})  
and two most related methods EM  \cite{bekker2019ecml} and LBE \cite{Gong2021} (described in Section \ref{related}).
We also consider ORACLE method  which assumes the full knowledge of a class variable $Y$ and calculates maximum likelihood estimator for the logistic fit. The ORACLE method serves as a reference method and obviously  it cannot be used in practise for PU data. To make a comparison fair, in TM, JOINT, EM and LBE we use the same stopping criterion, namely the convergence is reached when the relative change in the consecutive values of the objective function  used in all methods to estimate $\tilde \beta$, is less than $10^{-6}$ or the number of iterations exceeds $1000$.
Importantly, all methods are based on logistic model and thus the comparison between them is reliable.

In each experiment, the data sets are randomly split into training set ($70\%$) and testing set ($30\%$). We repeat the experiments $100$ times and average the results. Two evaluation measures are considered: accuracy calculated on the testing set and approximation error defined as
\[
AE=n_{\textrm{test}}^{-1}\sum_{i=1}^{n_{\rm test}}\vert y(x_{i},\hat{\beta})-y(x_{i},\hat{\beta}_{O})\vert , 
\]
where $\hat{\beta}$ is the solution of the considered   method, $\hat{\beta}_{O}$ is the solution of the ORACLE method and $n_{\textrm{test}}$ is the size of the testing set. In the case of artificial data sets, we replace $y(x_{i},\hat{\beta}_{O})$ by the true posterior probability $P(Y=1\vert X_i)$. The AE shows how close are the predictions of the considered methods to the predictions of ORACLE method (or to the true posterior probabilities for artificial datasets). Small value of the AE indicates that the  performance of the considered  PU method  is similar to that of   the oracle. In the experiments,  we observe that for some datasets the differences between the methods are more pronounced for AE than for standard accuracy.    

There are several important questions which are addressed in the experiments. First, how much do we lose compared to the oracle method and how much the proposed methods improve the prediction accuracy of the naive approach? Secondly,  what is the impact of various labelling schemes and how the performance deteriorates with decreasing label frequency? Thirdly, how robust are the considered methods against deviations from the logistic model? 
\subsection{Data sets}
We consider two artificial data sets: Artif1 and Artif2. They are obtained as follows. We first generate feature vector $X\sim N(0,I)$, where $I$ is $p\times p$ identity matrix. Then we pick the true class variable from Bernoulli distribution with success probability
$
P(Y=1\vert X)=F(X^{T}\beta),
$
where $\beta=p^{-1/2}(1,\ldots,1)^T$ and $F$ denotes the probability distribution function (PDF). We consider two forms of $F$: in the case of Artif1 we use PDF of the standard logistic distribution ($F(s)=\sigma(s)$) and in the case of Artif2 we use PDF of the standard Cauchy distribution (location $0$ and scale $1$). Since all considered methods are based on logistic model, Artif1 corresponds to the correct specification of the fitted model, whereas Artif2 corresponds to the misspecified model. 
Note that as Cauchy distribution significantly differs from logistic distribution in the tails,  Artif2 is strongly misspecified.
The advantage of using artificial datasets is that the true posterior probability is known (as opposed to the real datasets) and therefore it is possible to assess how accurately the considered methods estimate the posterior probability.  
In addition to artificial datasets, we consider $8$ benchmark datasets from UCI Machine Learning Repository. A short summary of each dataset can be found in Table 1 in Appendix 2. They were chosen to account for variable  characteristics of data (number of observations, number of features and difficulty of the classification problem). The  penultimate column of the  Table 1 contains the values of $R^2$ (proportion of explained deviance) calculated for the ORACLE method; the $R^2$ describes the goodness of fit of the model (the larger the better) and can be treated as a measure of the difficulty of classification problem.  For the considered datasets $R^2$ ranges from $0.23$ to $0.93$.

\subsection{Labelling scenarios}
We consider three methods of generating observed target variable $S$ based on the true target variable $Y$. The first scenario corresponds to SCAR assumption whereas for the two remaining schemes SCAR is violated. The second scenario is similar to the one considered in \cite{Gong2021} where the logistic sigmoid function is  also used to describe the propensity score. The last  scenario was considered in \cite{bekker2018phd} and is similar to the one considered in \cite{bekker2019ecml}.
\begin{enumerate}
\item \textbf{Scenario 1.} We consider constant propensity score function $e(x)=P(S=1\vert Y=1,X=x)=c$, where $c$ is label frequency which varies in simulations (SCAR scenario).
\item \textbf{Scenario 2.} Logistic propensity score function $e(x)=\sigma(x^T\gamma)$ is considered, where $\gamma= p^{-1/2}(g,\ldots,g)^{T}$ and $g$ is a parameter which varies in simulations. Obviously, for $g=0$ the scenario reduces to scenario 1 with $c=0.5$.
\item \textbf{Scenario 3.} Propensity score function is defined as $e(x)=\prod_{j=1}^{k}[sc(x(j),p^-,p^+)]^{1/k}$, where $x(j)$ is $j$-th coordinate of $x$ and 
$
sc(x(j),p^-,p^+):=p^-+\frac{x(j)-\min x(j)}{\max x(j)-\min x(j)}(p^+-p^-)
$ and $k$ is a chosen integer smaller than $p$.
\end{enumerate}

\subsection{Results}

We first analyse labelling scenario 1. 
Figures \ref{Fig_Artif_ls1}, \ref{Fig_acc_bench_ls1} show the impact of label frequency $c$ on the accuracy and  the approximation error for labelling scheme 1; see also Fig. 1 in Appendix 2 which shows approximation errors for scenario 1 for benchmark datasets. As  expected, the performance deteriorates with decreasing $c$ and all the curves approach the oracle curve for $c\approx 1$. The proposed TM method is the clear winner for both artificial data sets and is among the best performing methods for most benchmark datasets. 
The advantage of TM becomes significant for small $c$, e.g. for Breast Cancer and spambase  data sets its  approximation errors for $c=0.25$ are around 50\% smaller than approximation errors for EM and LBE. Importantly, TM usually works better than two most related methods: EM and LBE as well as  TM SIMPLE, which indicates that the proposed method of the  propensity score estimation is crucial for the  improved performance.
Generally, the ranking of the method depends on particular dataset and value of~$c$.  For example, the LBE method is a clear winner for heart-c and  wdbc for larger $c$, but it works significantly worse than other methods for  Breast Cancer and spambase when $c$ is small. On the other hand, the TM is a clear winner for these two datasets (Breast Cancer and  spambase) and small $c$, but it works much worse for adult dataset and larger $c$.
As expected, we observe the highest approximation errors for the NAIVE method which is due to the fact that it tends to underestimate $y(x)$ and the bias increases with decreasing $c$. Interestingly, we obtain similar results for Artif1 and Artif2 datasets, which shows that the methods are robust against model misspecification, even when the misspecification is strong. The JOINT method works worse than the TM which is probably associated with the optimization issues in the case for the JOINT method, which uses MM algorithm to search of the maximizer. There is certainly room for improvement in the JOINT method.
Tables 2 and 3 in Appendix 2 show the performance measures averaged over different $c=0.05,0.1,\ldots,0.95$.
We used the simple t-test to verify whether the difference between the winner method (in bold) and the second best is significant.
The last column contains the p-value of the t-test. Importantly, in most cases, the p-value is smaller than $0.05/15=0.003$. Note that $0.003$ is a level corresponding to Bonferroni correction as $0.05$ is a standard significance level and $15$ is the number of pairs among 6 considered methods. We remark that the proposed TM method has the lowest  rank averaged over data sets; it is a winner for 6/10 datasets (with respect to accuracy) and for 8/10 datasets (with respect to AE).

The results for different values of parameter $g$ for scenario 2 are presented in Figure \ref{Fig_Artif_ls2}. For artificial datasets, we observe stable results for all methods  (except the naive method), i.e. the accuracy and AE do not vary significantly for different values of parameter $g$. In addition, we can see the superior performance of the TM method. The TM exhibits the most stable performance among the methods. 
For naive method, the accuracy increases with $g$ and the AE decreases with $g$.  Tables 4 and 5 in Appendix 2 show the overall results for scenario 2 averaged over  different $g=0.1,0.2,\ldots,1$. 
Here, the LBE methods achieves on average the highest accuracy (it has the smallest averaged ranks), whereas the TM achieves on average the smallest approximation error.

Tables 6 and 7 in the Appendix 2 show the results for scenario 3, for $p^{-}=0.2$, $p^{+}=0.6$ and $k=5$.
The conclusions remain very similar to those of  the first and the second scheme. Namely, the TM method achieves the smallest approximation error (whereas the LBE is the second best); it is a winner for 8/10 datasets (wrt to AE) and for 5/10 datasetes (wrt accuracy).

In summary, we conclude that the averaged ranks over data sets  are the smallest for TM and LBE, with EM and TM SIMPLE being the third and the fourth method, respectively. This is a positive message, it means that the TM method (as well as LBE and EM) allows to estimate $y(x)$ accurately in a general case   when the SCAR assumption is not necessarily satisfied. 
Surprisingly, the TM SIMPLE method works quite well although it is based on a crude estimator of the propensity score function. For example it outperforms other methods for diabetes dataset and for scenario 2, for which the propensity score is not constant. As the TM SIMPLE is much faster than the remaining competitors (TM, EM, LBE and JOINT), it can be recommended in applications where reducing computational time is crucial. 

\begin{figure}[h]
\centering
$\begin{array}{cc}
\includegraphics[width=0.22\textwidth]{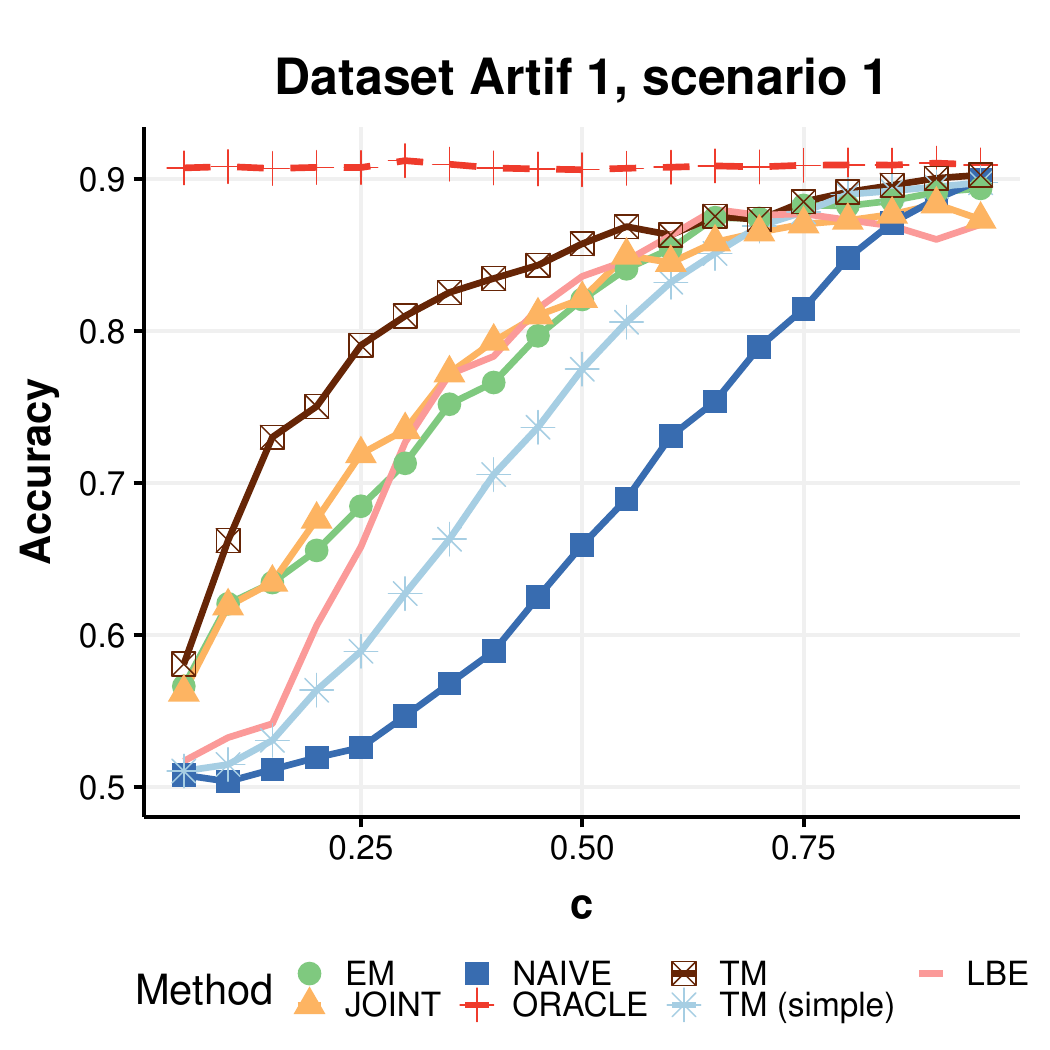} &
\includegraphics[width=0.22\textwidth]{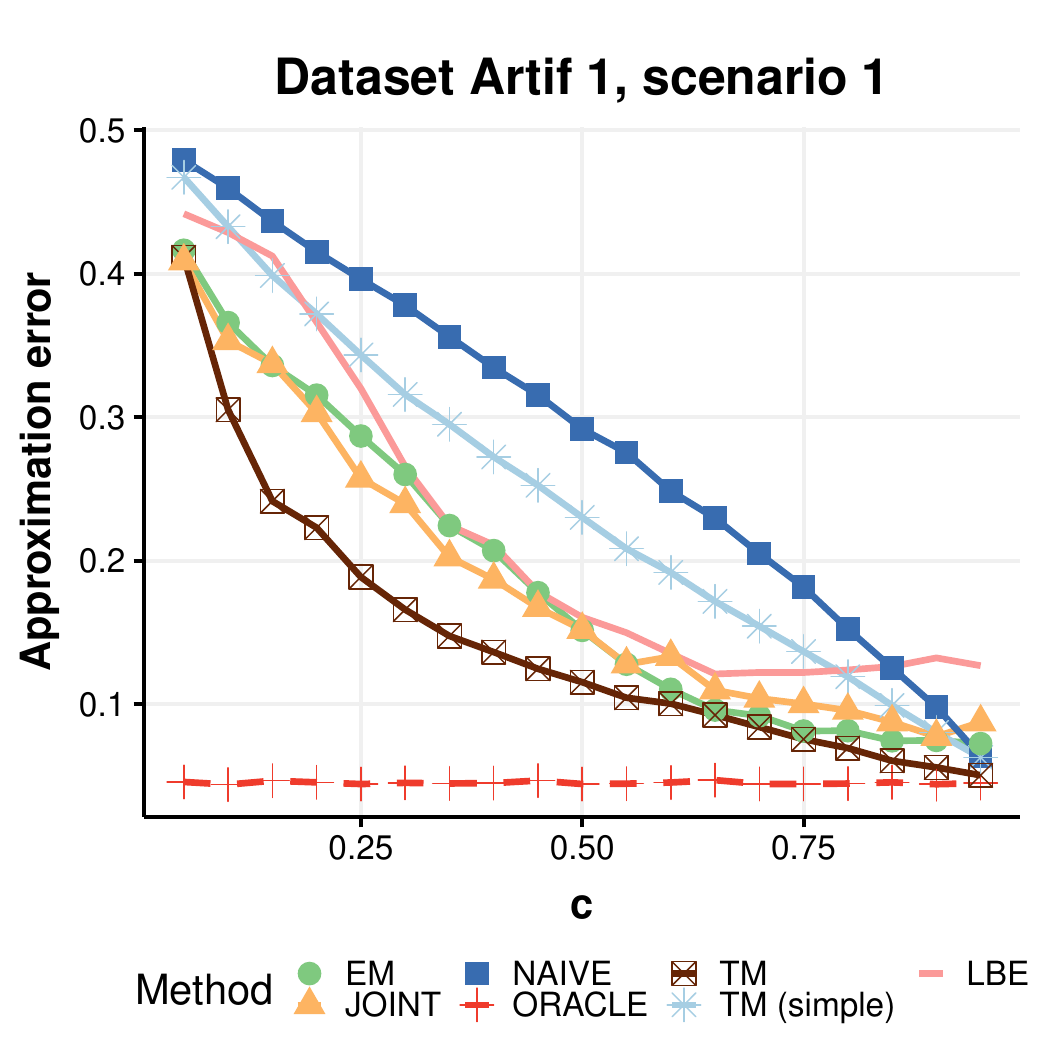}\\
\includegraphics[width=0.22\textwidth]{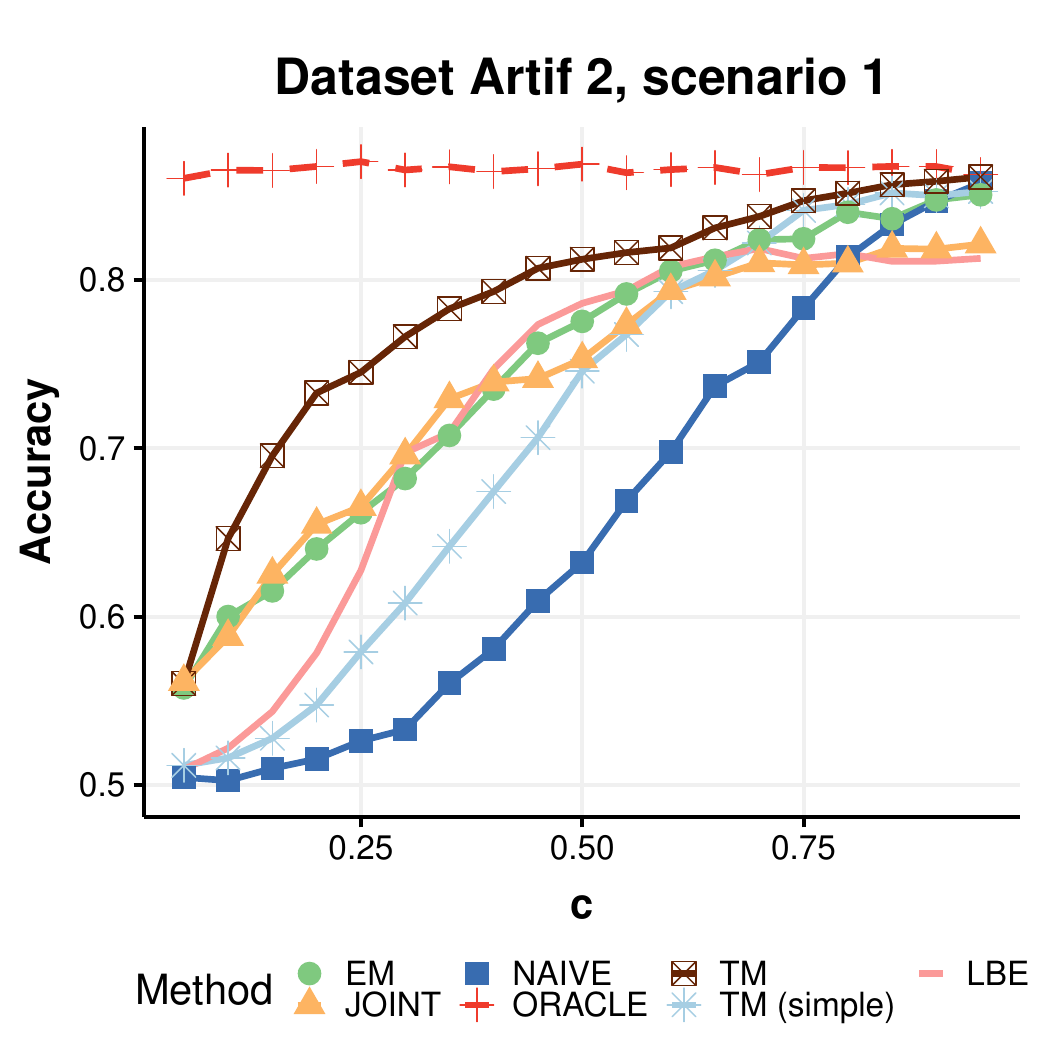}&
\includegraphics[width=0.22\textwidth]{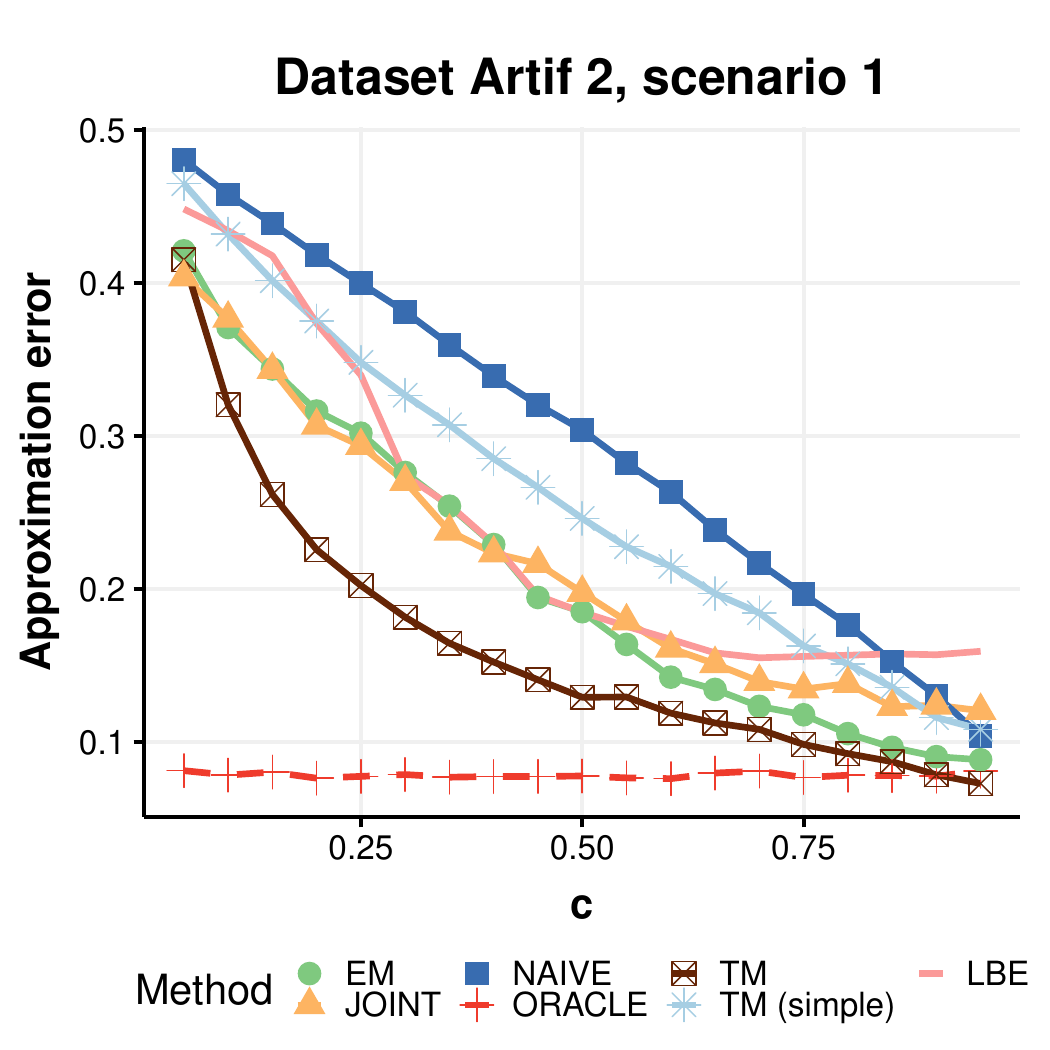} \\
\end{array}$
\caption{Accuracy and approximation errors for datasets Artif1 and Artif2 for scenario 1 and  different values of $c$.}
\label{Fig_Artif_ls1}
\end{figure}

\begin{figure}[h]
\centering
$\begin{array}{cc}
\includegraphics[width=0.22\textwidth]{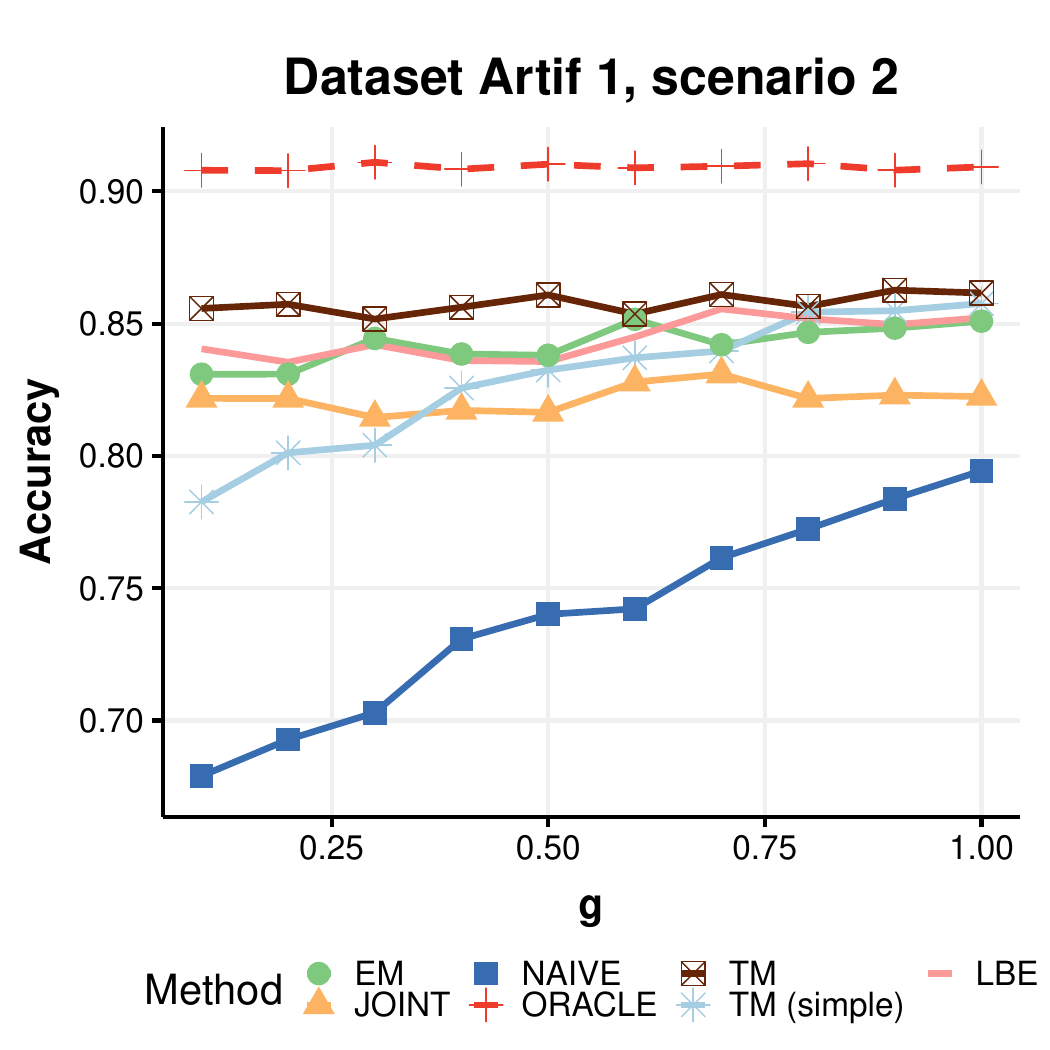} &
\includegraphics[width=0.22\textwidth]{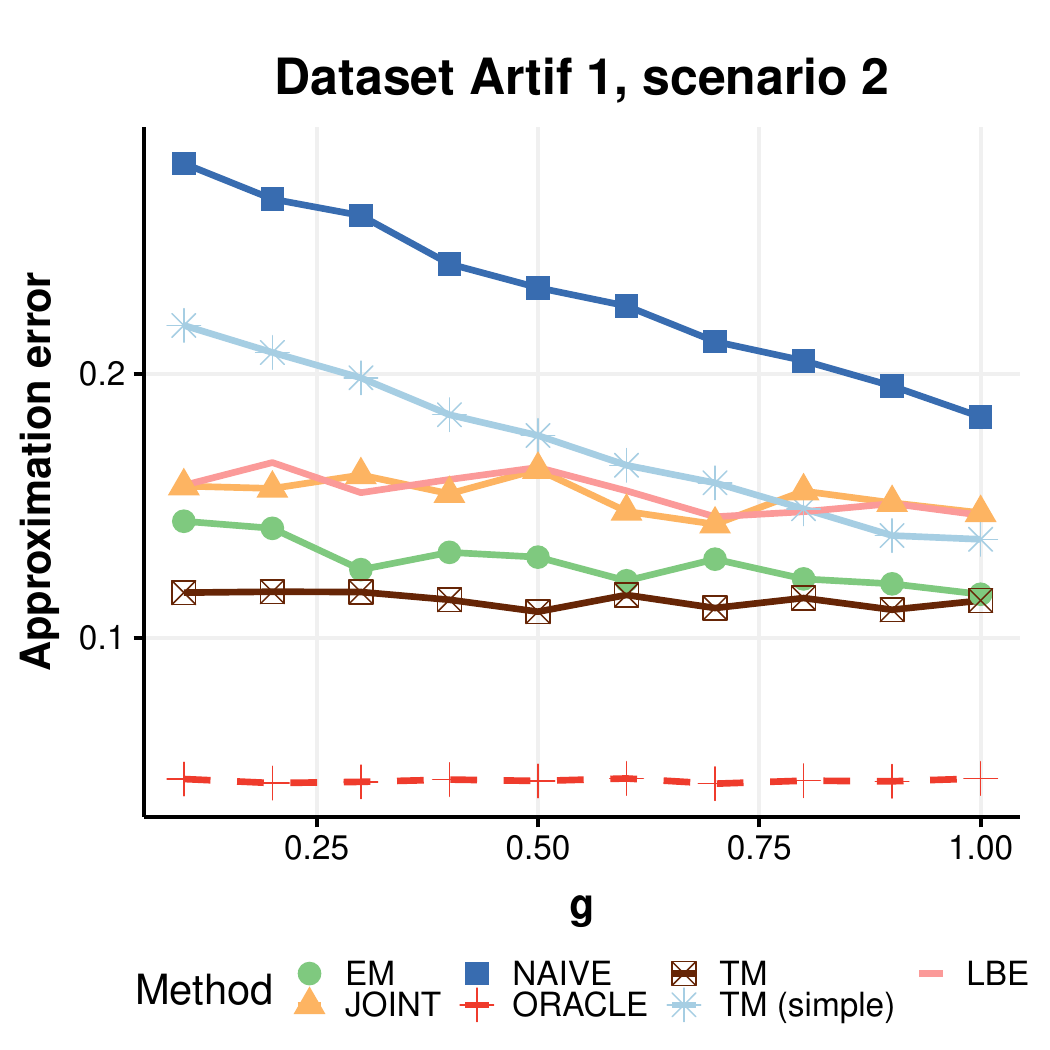}\\
\includegraphics[width=0.22\textwidth]{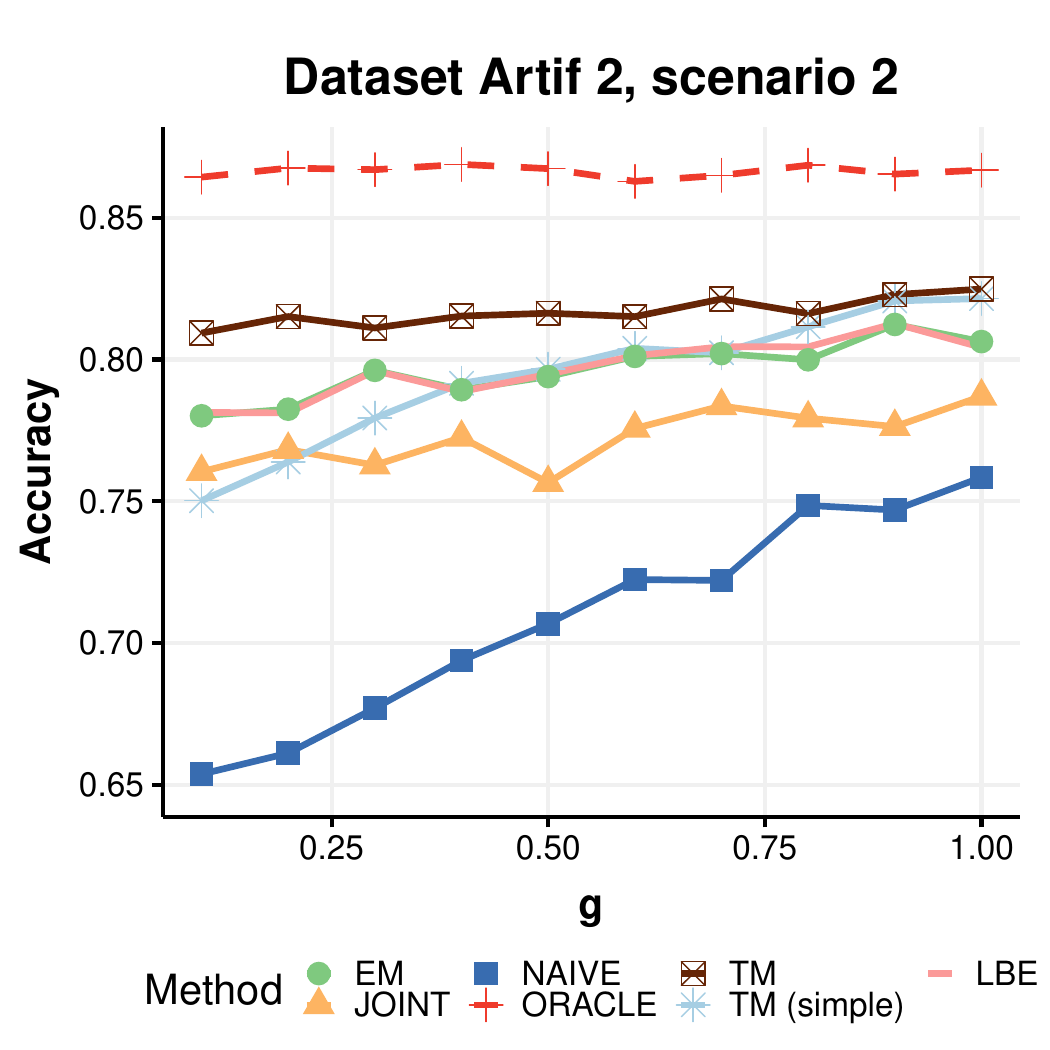}&
\includegraphics[width=0.22\textwidth]{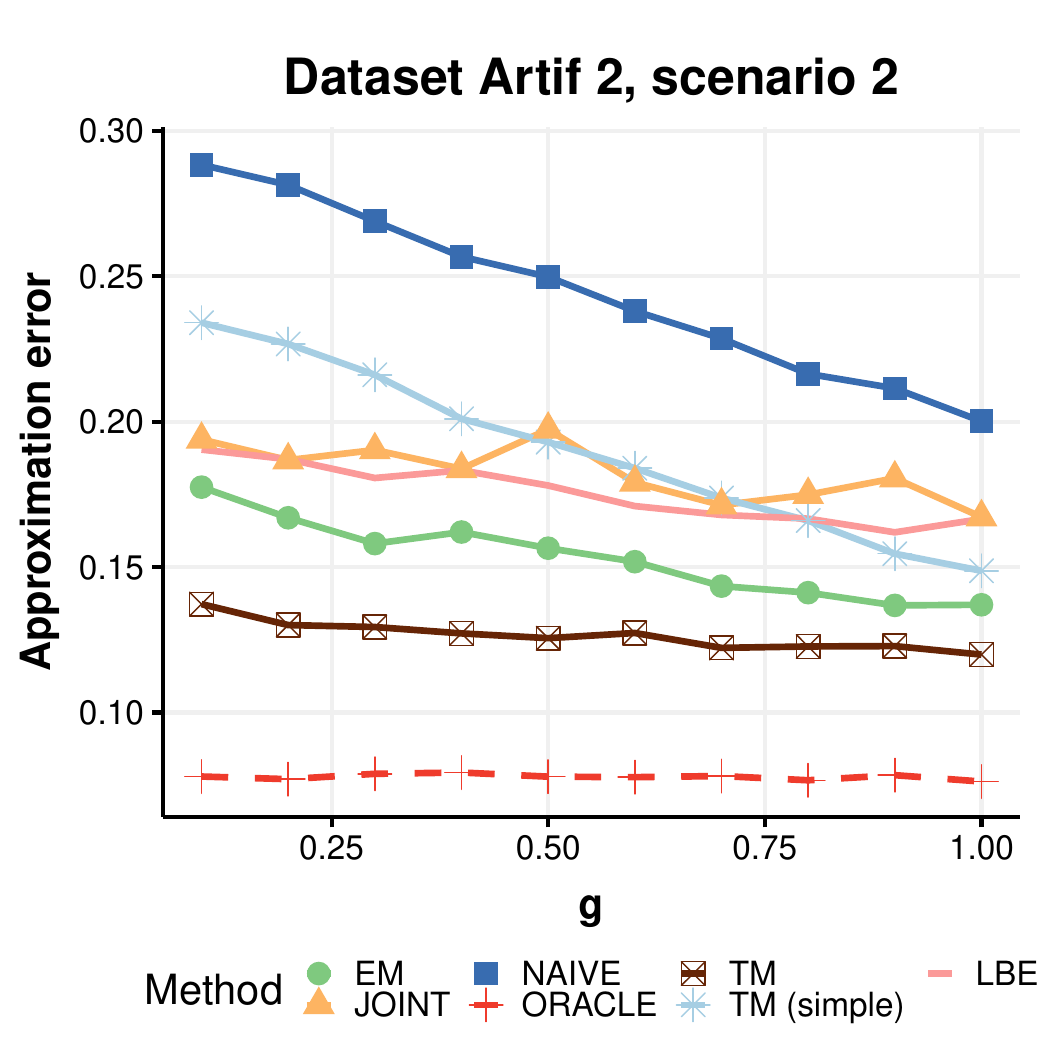} \\
\end{array}$
\caption{Accuracy and approximation errors for datasets Artif1 and Artif2 for scenario 2 and  different values of $g$.}
\label{Fig_Artif_ls2}
\end{figure}

 \begin{figure}[ht!]
 \centering
 $
 \begin{array}{cc}
 \includegraphics[width=0.22\textwidth]{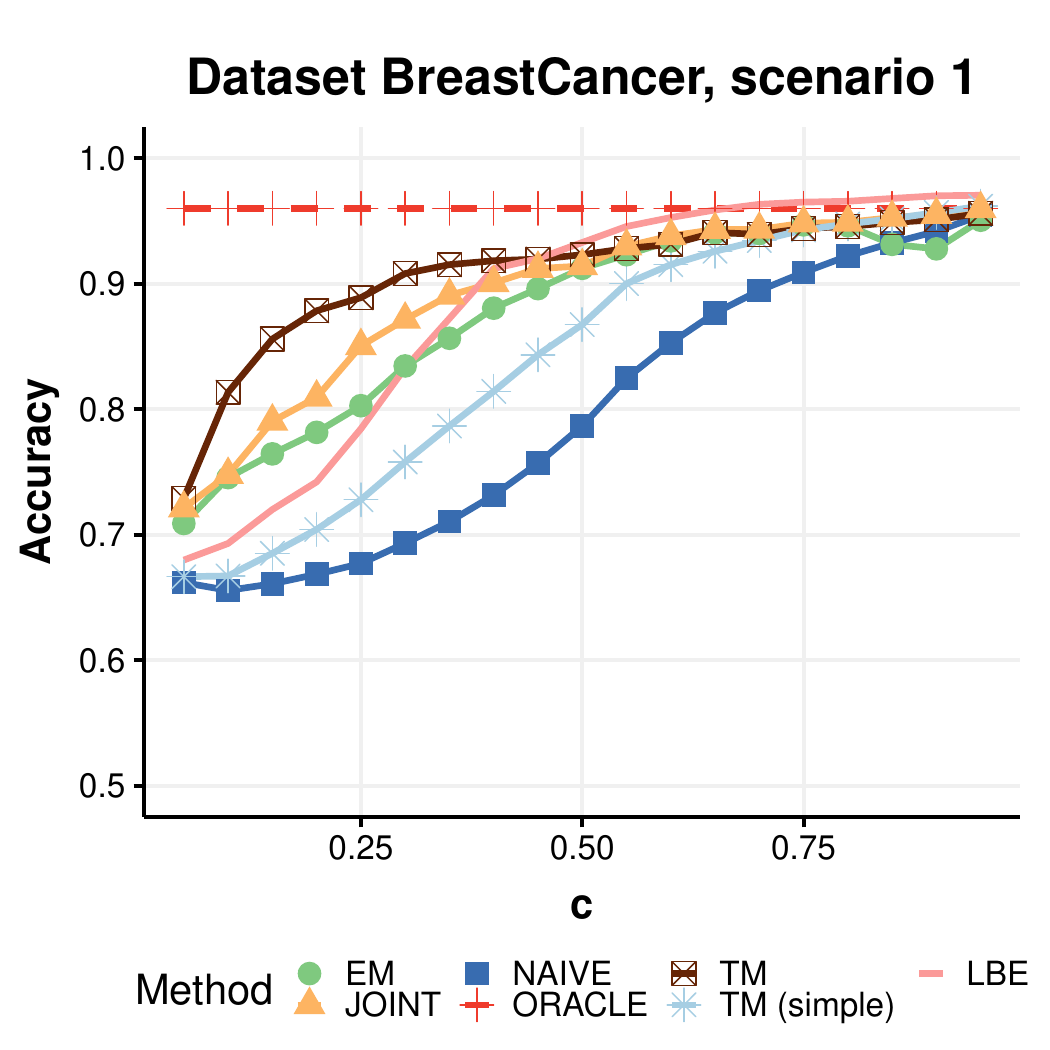} &
 \includegraphics[width=0.22\textwidth]{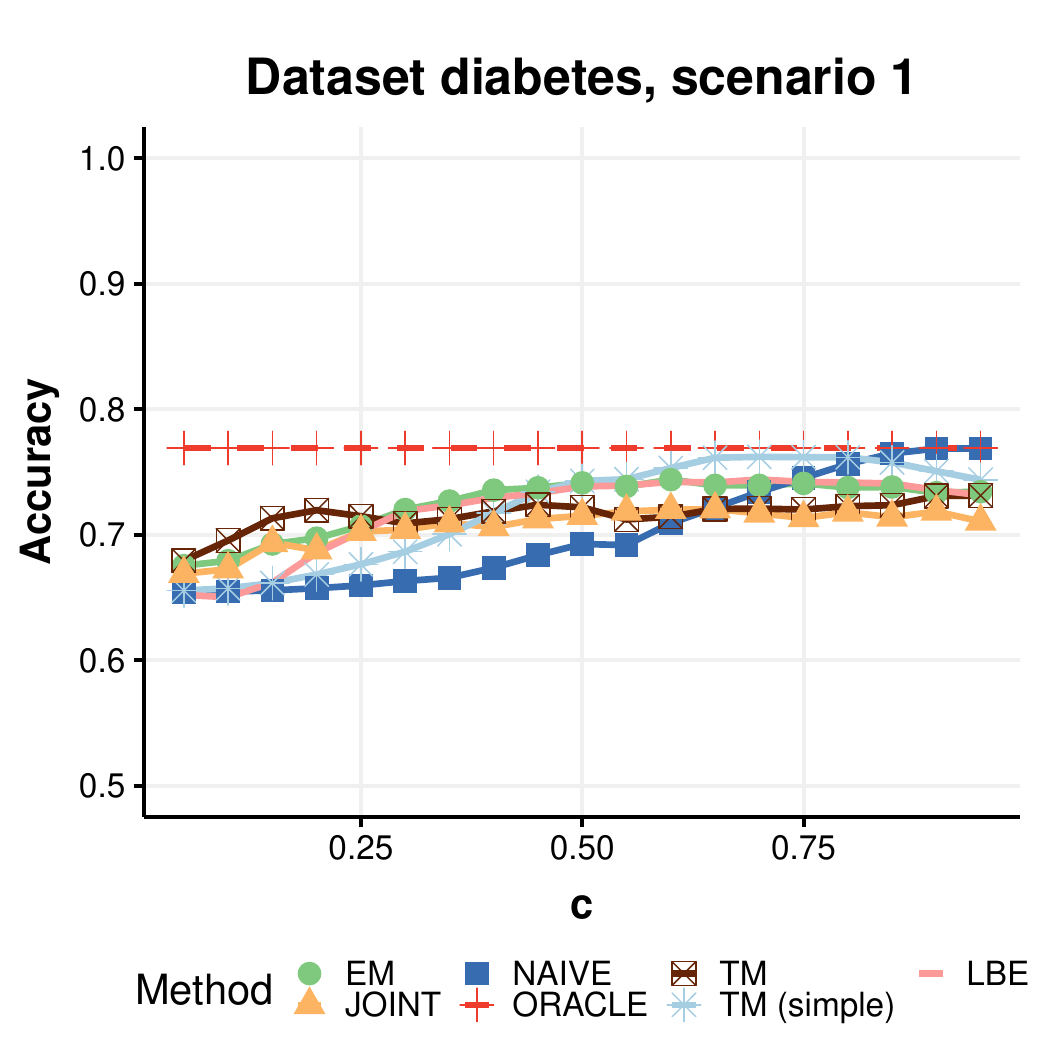}\\
 \includegraphics[width=0.22\textwidth]{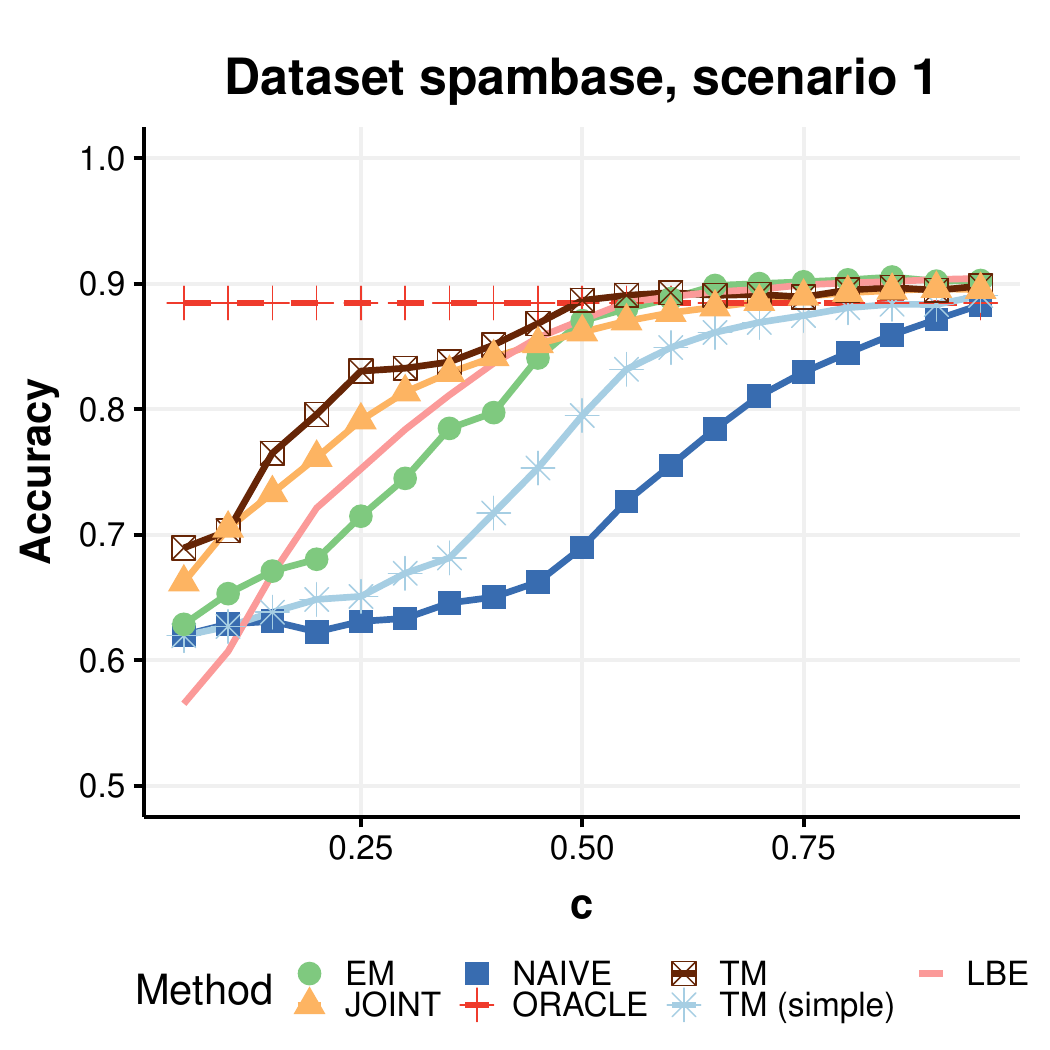} &
 \includegraphics[width=0.22\textwidth]{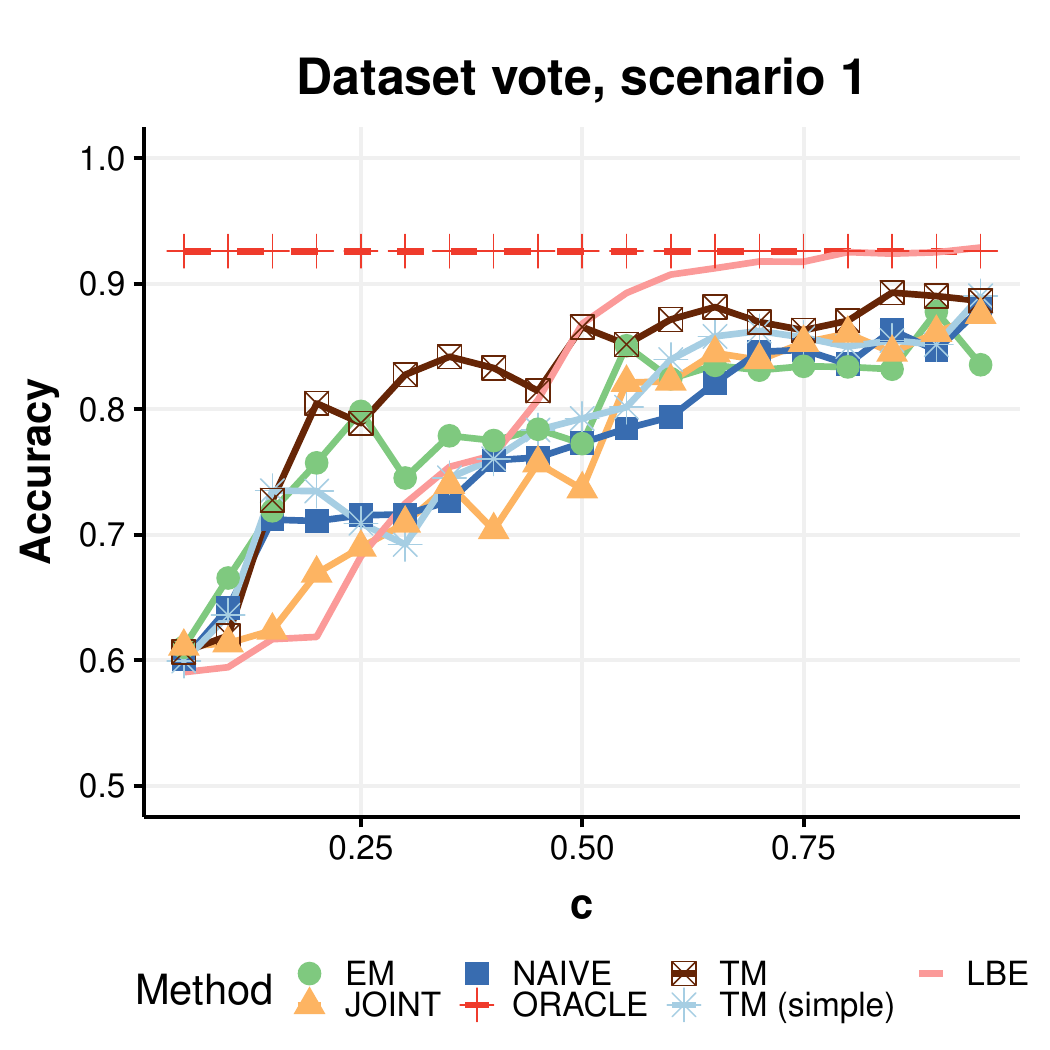} \\
 \includegraphics[width=0.22\textwidth]{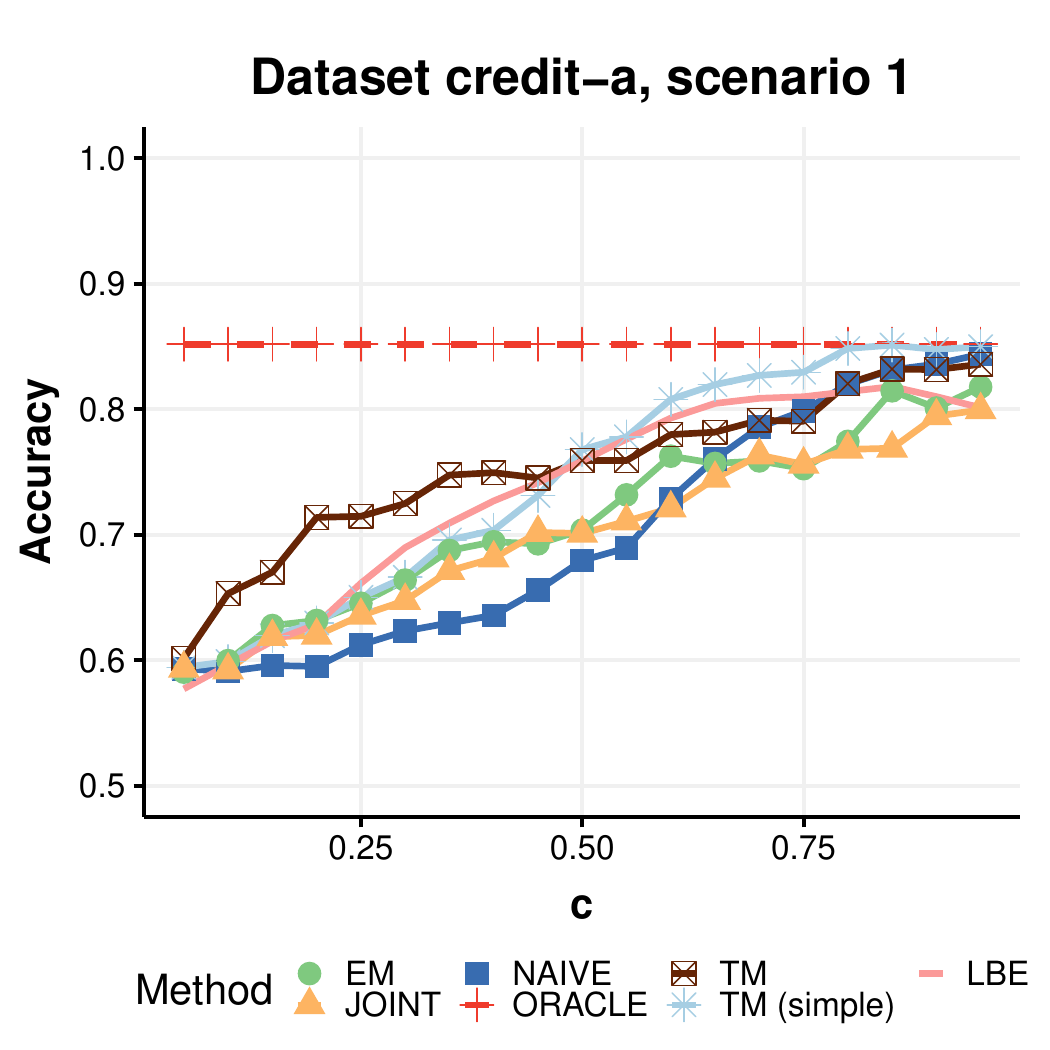} &
 \includegraphics[width=0.22\textwidth]{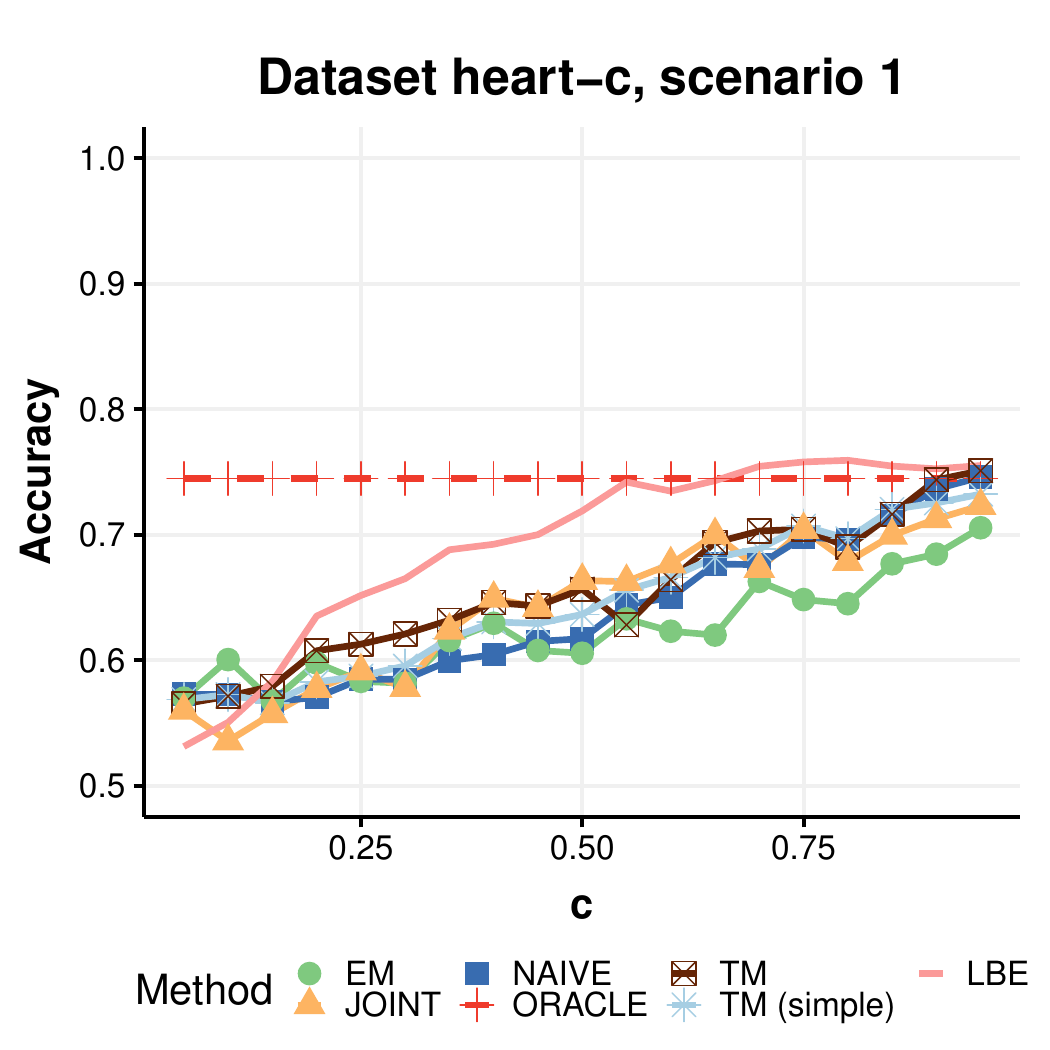} \\
 \includegraphics[width=0.22\textwidth]{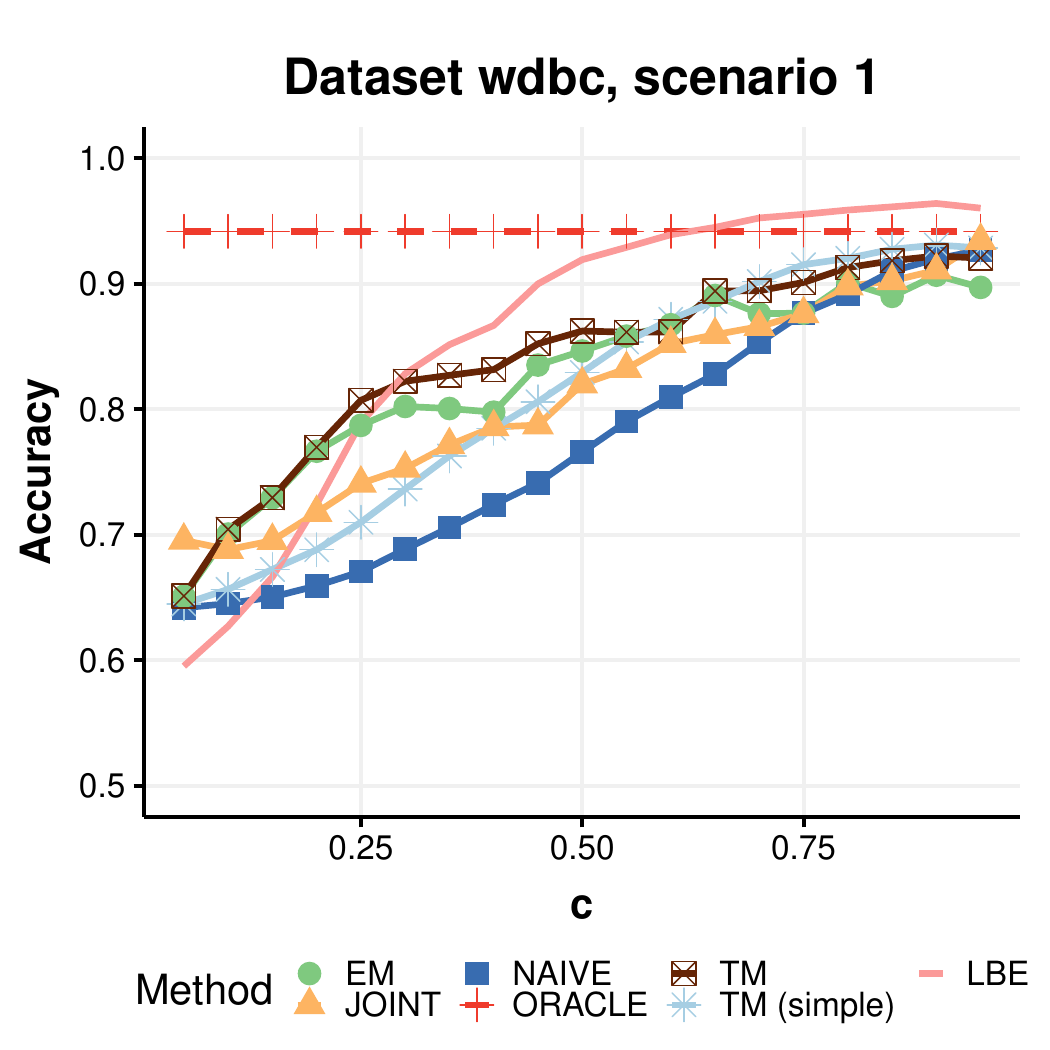} &
 \includegraphics[width=0.22\textwidth]{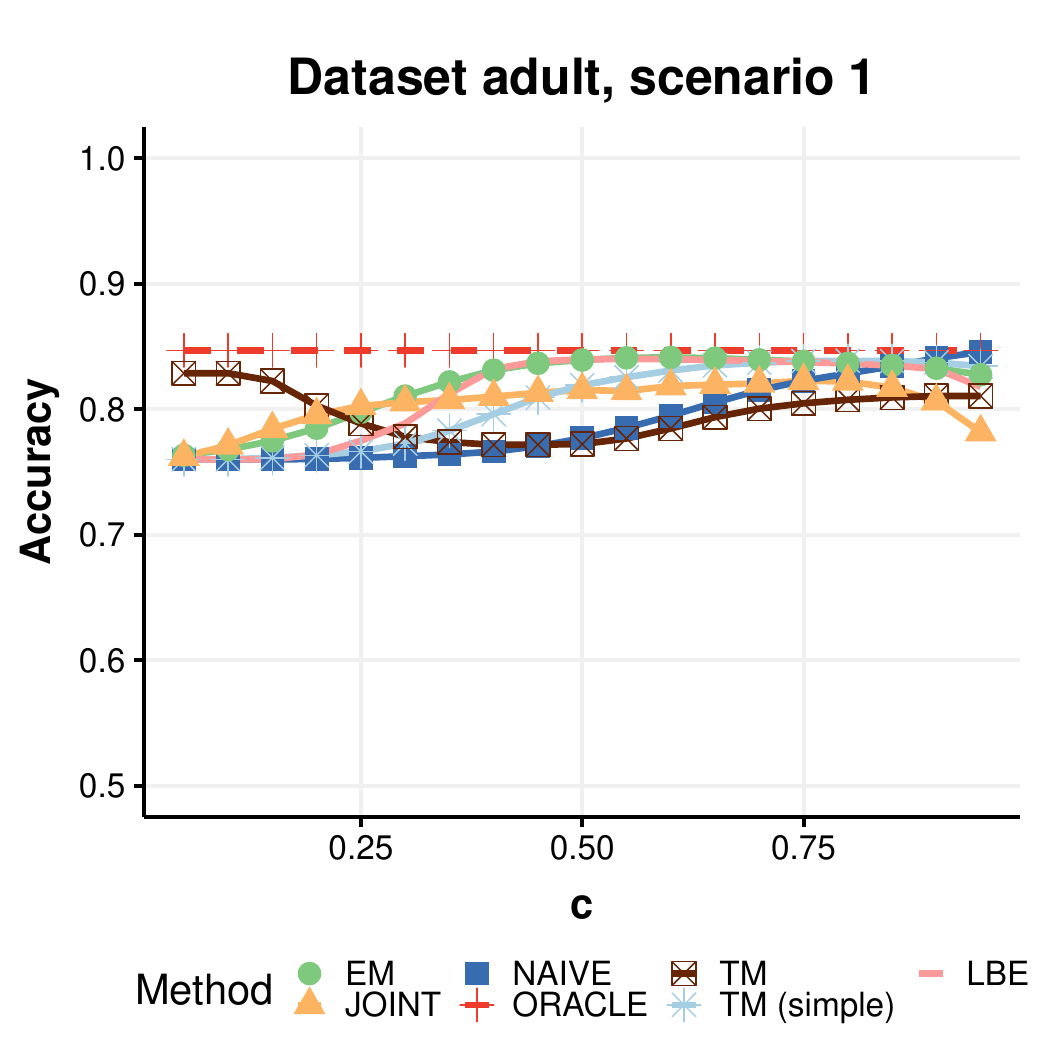} \\
 \end{array}$
 \caption{Accuracy for benchmark datasets for scenario 1 and  different values of $c$.}
 \label{Fig_acc_bench_ls1}
 \end{figure} 

\section{Conclusions} 
\label{Sec:Conclusions}
We have considered estimation problem for PU data when  SCAR assumption is not necessarily satisfied and have shown that when   posterior  probability $y(x)$ and propensity function $e(x)$ are both governed by the double  logistic model, their corresponding parameters $\tilde\beta^*$ and $\tilde\gamma^*$ are identifiable. This motivates JOINT method of estimation of 
$\tilde\beta^*$ and $\tilde\gamma^*$ by alternately maximising loglikelihood of the product logistic model. We have also proposed the second method, called TM, which relies on iterative maximisation of two estimated Fisher consistent expressions for the unknown parameters. For both approaches under certain assumptions we have proved consistency of the underlying estimators.
Analysis of their behaviour indicates that considering non-constant propensity function is crucial for estimation of the aposteriori probability as well as for performance of the  corresponding classifiers. In particular, the results show promising behaviour of TM method; it outperforms EM algorithm for most datasets and works on par with LBE method in the case of non-constant propensity score function. For constant propensity score, we observe the superior performance of the TM compared to the LBE.

There are still interesting issues left for future research. First, in addition to MM algorithm, other non-convex optimization procedures can be used in the JOINT method.This is important as its underperformance is likely caused by optimisation issues.
As for TM method, although the proposed method of estimation of the stratum $\mathcal{P}$  described in Section \ref{Sec:TM} works effectively, we believe that this crucial problem is worth further studying and there is still room for improvement. 
For the high-dimensional $X$ consideration of the regularised versions of the introduced methods is of interest. Moreover, note that the presented developments  open the way to test SCAR assumption, which under the considered model is equivalent to $\tilde\gamma=(\tilde\gamma,0^T)^T$.
Finally, the proposed methods can be possibly  adapted  to multi-label PU data, where multiple target variables are considered 
simultaneously and in many aplications, such as recommender systems, selection bias is frequent.


\ack The comments of  the reviewers, which
helped to improve the final version of  paper, are acknowledged. The research of W.R. was supported by the Polish National Science Center grant: NCN UMO-2018/31/B/ST1/00253.

\bibliography{References}
\end{document}